\documentclass[11pt, letterpaper]{article}
\usepackage[margin=.8in]{geometry}
\RequirePackage{fancyhdr}
\RequirePackage{xcolor} 
\RequirePackage{algorithm}
\RequirePackage{algorithmic}
\RequirePackage{natbib}
\RequirePackage{eso-pic} 
\RequirePackage{forloop}
\RequirePackage{url}
\RequirePackage{caption}
\usepackage{wrapfig}

\usepackage{microtype}
\usepackage{graphicx}
\usepackage{subfigure}
\usepackage{booktabs} 

\usepackage[dvipsnames]{xcolor}
\usepackage[colorlinks=true]{hyperref}
\definecolor{mydarkblue}{rgb}{0,0.08,0.45}
\hypersetup{%
,urlcolor=mydarkblue
,citecolor=mydarkblue
,linkcolor=red
}
\usepackage{url}

\usepackage{amsmath,amsfonts,bm}
\usepackage{amssymb}
\usepackage{mathtools}
\usepackage{bbm}

\usepackage[capitalize,noabbrev]{cleveref}

\usepackage{amsthm}
\theoremstyle{plain}

\usepackage{amsthm}

\newtheorem{lem}{Lemma}

\newtheorem{prop}{Proposition}

\def\calL{{\mathcal{L}}}

\def\calX{{\mathcal{X}}}
\def\calY{{\mathcal{Y}}}

\newcommand{\E}{\mathbb{E}}

\DeclareMathOperator*{\argmax}{arg\,max}

\newcommand{\regret}{\mathsf{regret}}

\begin{document}
\title{Best-of-Tails: Bridging Optimism and Pessimism\\in Inference-Time Alignment}
\date{}
\author{
    Hsiang~Hsu\thanks{Correspondence to: Hsiang Hsu \texttt{<hsiang.hsu@jpmchase.com>}.}, Eric~Lei, and~Chun-Fu~(Richard)~Chen\\
    JPMorganChase Global Technology Applied Research
}

\maketitle

\begin{abstract}
    Inference-time alignment effectively steers large language models (LLMs) by generating multiple candidates from a reference model and selecting among them with an imperfect reward model. However, current strategies face a fundamental dilemma: ``optimistic'' approaches like Best-of-$N$ suffer from reward hacking, while ``pessimistic'' regularized methods often stifle the exploration needed to discover high-quality responses. In this work, we formalize this trade-off through the lens of regret minimization, demonstrating that the optimal strategy depends critically on the tail behavior of the reward distribution. We show theoretically that light-tailed regimes favor optimism to unearth high-quality outliers, whereas heavy-tailed regimes require pessimism to guard against reward mis-calibration in the extremes. Guided by this insight, we introduce Best-of-Tails (BoT), an adaptive inference-time alignment framework that uses Tsallis divergence as a tunable regularizer to provide a finer granularity of interpolation between these extremes. BoT uses the Hill estimator to characterize reward-tail heaviness on a per-prompt basis and dynamically adjusts its selection rule to balance exploration gains against alignment error. Across math, multiple-choice reasoning, and human-preference evaluations, BoT improves alignment performance across a range of reference and reward model configurations relative to fixed-strategy baselines.
\end{abstract}


\section{Introduction}
Inference-time scaling has emerged as a key mechanism for adapting large language models (LLMs) to complex real-world tasks \citep{brown2024large, snell2024scaling}. 
Extending traditional scaling laws \citep{kaplan2020scaling} to the inference phase, models are allocated additional compute at deployment---effectively giving them more time to ``think''---so they can search a richer response space and select outputs that are more accurate and contextually appropriate. 
This paradigm has helped unleash latent capabilities of LLMs, including improved reasoning through chain-of-thought (CoT) generation \citep{wei2022chain, feng2023towards}, iterative self-correction and self-evaluation \citep{wu2025meta, ren2023self, asai2024self}, and longer-horizon planning via structured search \citep{zhang2024rest, yao2023tree}.

The principles of inference-time scaling have also found a natural and increasingly important application in LLM alignment, where the goal is to steer model's behavior toward human preferences such as correctness, helpfulness, and safety \citep{christiano2017deep, schulman2017proximal, rafailov2023direct}. 
Rather than permanently encoding preferences into model weights through fine-tuning, inference-time alignment uses extra computation at inference to evaluate and select among multiple candidates produced by a reference model, guided by a reward model (RM) \citep{stiennon2020learning, jinnai2024regularized, verdun2025soft, gui2024bonbon, geuter2025guided, huang2025best, khalaf2025inference}. 
This approach offers two practical advantages: it avoids the expense and operational burden of repeated re-training, and enables rapid adaptation to new tasks and deployments by simply changing the reward objective.

\begin{figure}[!tbh]
  \begin{center}
    \centerline{\includegraphics[width=.7\textwidth]{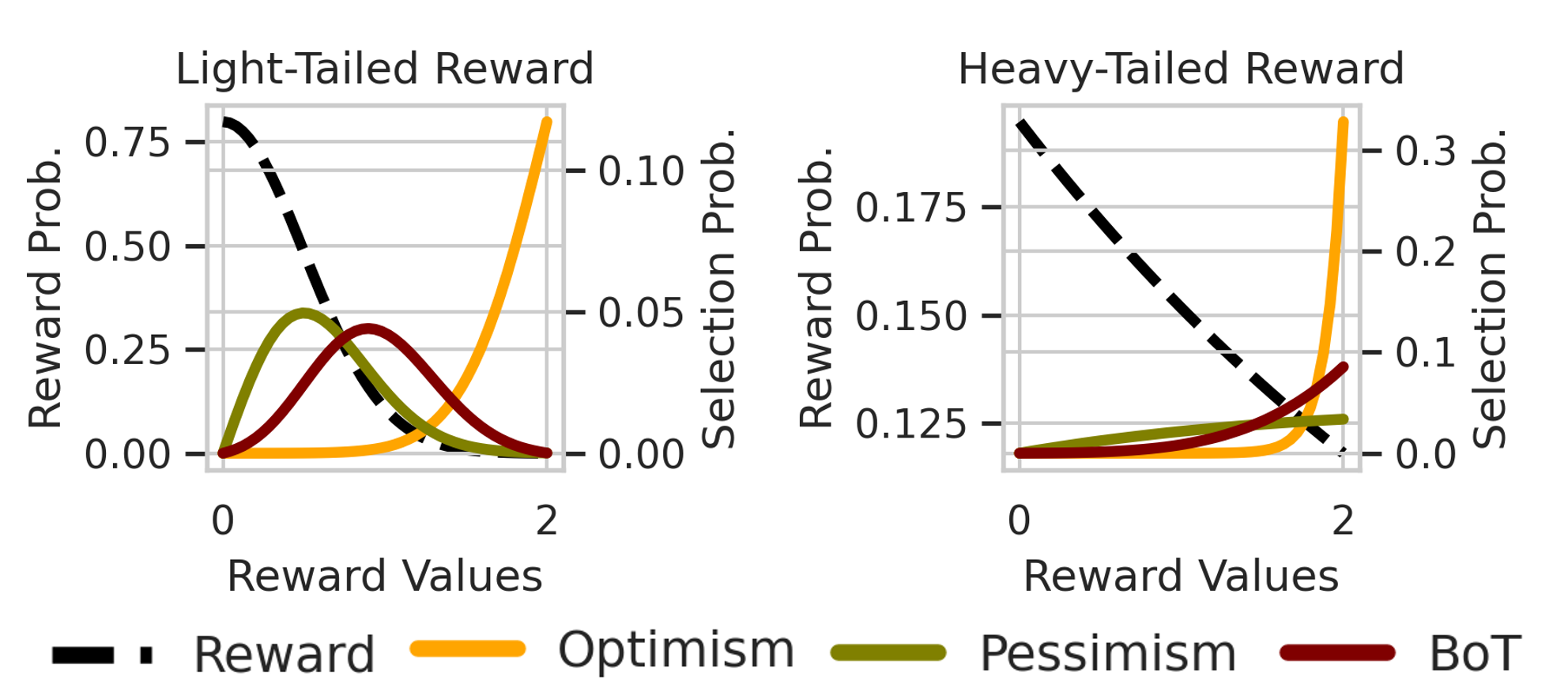}}
    \caption{
    Conceptual illustration of selection probabilities (colored solid lines) for optimistic, pessimistic, and the proposed \texttt{BoT} strategies.
    The plots depict how these strategies re-weight candidates under light-tailed (\textbf{left}) versus heavy-tailed (\textbf{right}) reward distributions (black dashed lines).
    While optimism consistently concentrates mass on the highest rewards (risking reward hacking) and pessimism remains conservative (risking under-exploration), \texttt{BoT} adaptively shifts its strategy: it mimics optimism in light-tailed regimes to exploit safe gains, but pivots toward robust, conservative selection in heavy-tailed regimes to prevent over-optimization.
    }
    \label{fig:intro}
  \end{center}
\end{figure}

The most prominent inference-time alignment strategy is Best-of-$N$ (\texttt{BoN}), which selects the candidate with the highest reward from $N$ generated responses \citep{stiennon2020learning}.
At its core, \texttt{BoN} reflects an ``optimistic'' premise: that improvements in reward scores translate reliably into improvements in true quality.
In practice, however, reward models are only imperfect proxies \citep{laidlaw2024correlated, sun2024rethinking}.
As $N$ increases, \texttt{BoN} is increasingly driven by the extreme tail of the reward distribution, precisely where reward scores can become systematically mis-calibrated and over-estimate true quality, leading to reward hacking\footnote{Also known as reward over-optimization or Goodhart’s Law.} \citep{stiennon2020learning, gao2023scaling, skalse2022defining, kwa2024catastrophic}. 
To mitigate this failure mode, recent work has proposed more ``pessimistic'' alternatives that explicitly curb over-optimization through regularization or conservative selection rules \citep{jinnai2025regularized, huang2024correcting, huang2025best, verdun2025soft}. 
Yet, these safeguards often come at a cost, stifling the model's exploration at inference time to discover genuinely outstanding solutions when the reward signal is informative.
We visually illustrate this dichotomy in Figure~\ref{fig:intro}, which demonstrates how optimistic and pessimistic strategies select the responses depending on the observed reward tail heaviness.

In this paper, we formalize the central trade-off between optimistic and pessimistic inference-time alignment strategies. 
Specifically, we argue that the effectiveness of inference-time alignment depends not only on how accurately the reward model reflects true response quality, but also on the distribution of reward scores induced by the generated candidates. 
In \S~\ref{sec:tail-softbon-itp}, we analyze inference-time alignment through the lens of regret minimization.
Our theoretical results reveal that the reward score tail behavior fundamentally shapes what ``good'' selection looks like.
When rewards are largely concentrated—corresponding to a light-tailed regime—overly pessimistic selection can underutilize inference compute, since aggressive conservatism makes it unlikely to surface meaningfully better responses. 
In this setting, optimistic strategies such as \texttt{BoN} can be effective by actively exploring candidates with higher proxy reward. 
In contrast, when reward scores exhibit heavy tails—indicating substantial headroom for exploration---\texttt{BoN} is more likely to chase extreme proxy scores where mis-calibration is severe, making it particularly vulnerable to reward hacking; here, more pessimistic strategies become essential for robustness.

Guided by this insight, in \S~\ref{sec:BoT} we propose Best-of-Tails (\texttt{BoT}), a new framework that bridges the optimism needed to discover high-reward responses and the pessimism required to ensure their validity. 
\texttt{BoT} introduces Tsallis divergence \citep{tsallis1988possible, lee2019tsallis}, a generalized family of divergences that interpolates between the KL and $\chi^2$ divergences, as a tunable regularization term for response selection. 
By our design, \texttt{BoT} does not rely on a fixed hyperparameter; instead, it first characterizes the reward landscape by estimating the tail index (power coefficient) via the Hill estimator \citep{hill1975simple}.
Based on this estimate, \texttt{BoT} adaptively interpolates between optimistic and pessimistic selection strategies for each input prompt (cf.~Figure~\ref{fig:intro}), offering a robust trade-off between exploration cost and alignment error.
In \S~\ref{sec:experiments}, we empirically demonstrate that \texttt{BoT} consistently outperforms fixed optimistic baselines (e.g., \texttt{BoN}) and fixed pessimistic alternatives \citep{jinnai2024regularized, verdun2025soft, huang2025best} across math, multiple-choice reasoning, and human-preference evaluations, under multiple reference and reward model configurations.
We conclude in \S~\ref{sec:conclusion} with a discussion of limitations and future research directions.

Omitted proofs, additional discussions, experiment setups, and additional experiments are included in the Appendix. 
Code to reproduce our experiments will be released soon. 

\section{Taxonomy of Inference-Time Alignment through Regret Analysis}\label{sec:related-work}
Given a prompt $x \in \mathcal{X}$, an LLM is represented as a (stochastic) policy $\pi: \mathcal{X} \to \Delta(\mathcal{Y})$ that induces a conditional distribution over responses $y \in \mathcal{Y}$. 
In other words, $\pi(y | x)$ specifies the probability of generating a response $y$ conditioned on $x$.
We consider a true (but unknown) reward function $r^*: \mathcal{X} \times \mathcal{Y} \to [0, R_{\textsf{max}}]$ with a bounded range $R_{\textsf{max}}$.
For a prompt $x$, the true reward $r^*$ quantifies the extent to which $y$ aligns with human preference, passes a verifier, or demonstrates sound reasoning. 
In practice, $r^*$ is inaccessible, and an imperfect proxy $\hat{r}$ is often employed instead. 
For a given reward function $r \in \{r^*, \hat{r}\}$, the corresponding expected reward under policy $\pi$ is defined as $J_r(\pi;x) = \E_{y\sim\pi(\cdot| x)}[r(x, y)]$.
Moreover, letting $\|\cdot\|_{L^p(Q)}$ denote the $\calL^p$-norm with respect to a distribution $Q$, we measure the quality of $\hat{r}$ via its expected error with respect to $r^*$ under a policy: $\epsilon_p(\hat{r}; \pi(\cdot|x)) = \|\hat{r}-r^*\|_{L^p(\pi)} = \left(\E_{y\sim \pi(\cdot| x)}[|\hat{r}(x, y)-r^*(x, y)|^p]\right)^{1/p}$.
For simplicity, we denote $\epsilon = \epsilon_2(\hat{r}; \pi_\textsf{ref}(\cdot|x))$.
Finally, let $\mathbbm{1}(\cdot)$ denote the indicator function and $[b]_+ = \max(b, 0)$ for a scalar $b$.
We define the $\alpha$-exponential as $\exp_\alpha(x) = [1+(\alpha-1)x]_+^{\frac{1}{\alpha-1}}$ and $\alpha$-logarithm\footnote{The $\alpha$-exponential and $\alpha$-logarithm converge to the standard exponential and logarithm functions as $\alpha \to 1$ (Lemma~\ref{app:lem:alpha-exp-log}).} as $\log_\alpha(x) = \frac{x^{\alpha-1}-1}{\alpha-1}$.

\textbf{Inference-time alignment.}
Ideally, LLM alignment seeks a policy that maximizes the expected true reward: $\pi^*(\cdot | x) = \argmax_{\pi \in \Delta(\calY)} J_{r^*}(\pi; x)$.
However, since the true reward $r^*$ is typically inaccessible, practitioners must optimize a proxy objective $J_{\hat{r}}(\pi; x)$ instead.
To prevent the model from over-exploiting the imperfect proxy, it is standard to include an explicit distortion penalty. 
For instance, given a $\lambda \geq 0$, reinforcement learning from human feedback (RLHF) \citep{christiano2017deep} uses KL regularization:
\begin{equation}\label{eq:rlhf}
    \max_{\pi \in \Delta(\calY)} J_{\hat{r}}(\pi; x) - \lambda D_\textsf{KL}(\pi|\pi_\textsf{ref}),
\end{equation}
which limits how far the aligned policy can deviate from the reference policy.
Yet, optimizing this objective via iterative parameter updates is often infeasible, as modern LLMs are frequently accessible only as black boxes (e.g., via commercial APIs) or are prohibitively expensive to fine-tune. 
These limitations motivate the study of \emph{inference-time alignment}, which aims to steer model behavior dynamically at inference time without updating model parameters.

An inference-time alignment method begins with a reference policy $\pi_\textsf{ref}$ that is capable of generating high-reward responses \emph{by chance}.
The method then steers the output distribution by \emph{re-weighting} these responses using a weight function $w(y | x)$, which is typically designed to be monotonically increasing in the proxy reward $\hat{r}(x,y)$.
At the distribution level, this induces a re-weighted policy
\begin{equation}
    \pi_w(y|x) \propto \pi_\textsf{ref}(y|x)w(y|x),
\end{equation}
which can be interpreted as an ``aligned'' version of $\pi_\textsf{ref}$ that favors higher proxy-reward responses.
Since direct access to the full reference distribution is generally unavailable, inference-time alignment procedures operate on empirical samples.
Consequently, the procedure first draws a finite set of i.i.d. candidates $\calY_N = \{y_1, \cdots, y_N\} \sim \pi_\textsf{ref}(\cdot|x)$, and then approximate $\pi_w(y | x)$ using techniques such as weighted sampling \citep{verdun2025soft} or rejection sampling \citep{liu2023statistical, block2023sample}.
We denote this finite-sample approximation by $\hat{\pi}_w(y | x)$.
A final ``best'' response is then selected according to $\hat{\pi}_w(y | x)$, typically by sampling \citep{verdun2025soft, huang2025best} or by greedily choosing the highest-scoring candidate \citep{stiennon2020learning}.
This \emph{sample–evaluate–select} pipeline has become a standard framework for inference-time alignment \citep{huang2024self}.

\textbf{Inference-time regret analysis.}
The effectiveness of an inference-time alignment policy $\hat{\pi}_w(y | x)$ can be quantified by the loss in expected true reward caused by imperfect alignment under a finite sample budget $N$.
This quantity, termed the \emph{inference-time regret} \citep{huang2025best, aminian2025best}, is defined as 
\begin{equation}
    \textsf{regret}(\hat{\pi}_w; x) = J_{r^*}(\pi^*; x) - J_{r^*}(\hat{\pi}_w; x).
\end{equation}
Analyzing the upper bound of this regret cleanly isolates key factors that govern the design of effective inference-time alignment strategies.

\begin{prop}\label{prop:general-regret-upper-bound}
    For a given prompt $x\in\calX$, the inference-time regret for a general inference-time alignment policy $\hat{\pi}_w(y|x)$, defined via a re-weighting function $w(y|x)$, admits the following upper bound\footnote{A more general upper bound is provided in Appendix~\ref{app:proof:prop:general-regret-upper-bound}.}
    \begin{equation}\label{eq:alignment-regret-general}
    \begin{aligned}
          \regret(\hat{\pi}_w; x)&\leq 2R_{\textsf{max}}\underbrace{D_\textsf{TV}(\hat{\pi}_w(\cdot|x)\|\pi_w(\cdot|x))}_{(i)} - \underbrace{\Delta(\pi_w(\cdot|x))}_{(ii)}\\
          &\quad+ \underbrace{\sqrt{C^{\pi^*}(x)-1}}_{(iii)}\left(\underbrace{\epsilon_2(\hat{r}; \pi_\textsf{ref}(\cdot|x))}_{(iv)} + \|\hat{r}\|_{L^2(\pi_\textsf{ref}(\cdot|x))}\right)\\
          &\quad+\underbrace{\sqrt{1 + D_{\chi^2}(\pi_w(\cdot|x)\|\pi_\textsf{ref}(\cdot|x))}}_{(v)}\epsilon_2(\hat{r}; \pi_\textsf{ref}(\cdot|x)).
    \end{aligned}
    \end{equation}
    Here, $D_\textsf{TV}(\cdot\|\cdot)$ is the total variation \citep{lehmann2005testing}.
    $\Delta(\pi_w(\cdot|x)) = J_{\hat{r}}(\pi_w) - J_{\hat{r}}(\pi_\textsf{ref})$ represents the alignment gain, i.e., the improvement in expected proxy reward achieved by shifting from $\pi_\textsf{ref}$ to $\pi_w$.
    We denote by $C^{\pi^*}(x) = \E_{y\sim\pi^*(\cdot|x)}[\pi^*(y|x)/\pi_\textsf{ref}(y|x)]$ the coverage of the reference policy $\pi_{\textsf{ref}}$ with respect to the optimal policy $\pi^*$.
    Finally, $D_{\chi^2}(\pi_w(\cdot|x)\|\pi_\textsf{ref}(\cdot|x))$ measures the distortion of the aligned policy $\pi_w$ from $\pi_\textsf{ref}$ via the $\chi^2$ divergence.
\end{prop}
Apart from term (i), which arises solely from the finite-sample approximation of $\pi_w$ using the candidate set $\calY_N$, the remaining terms (ii)–(v) each capture a distinct axis along which inference-time alignment methods can be analyzed and improved.
Crucially, the bound in Eq.~\eqref{eq:alignment-regret-general} offers a more nuanced perspective than recent analyses (e.g., \citet{huang2025best}), which tend to be dominated by the oracle constant $C^{\pi^*}(x)$ and consequently overlook the structural constraints imposed by $\pi_\textsf{ref}$. 
To explicitly investigate how the reward tail governs alignment dynamics, we focus on the interplay between terms (ii) and (iv). 
Together, these terms mirror the core RLHF objective (Eq.~\eqref{eq:rlhf}), forming a trade-off to maximize reward while penalizing distortion.
We next iterate over these terms and survey the related work.

\textbf{Alignment gain (term (ii)).}
Empirically, Best-of-$N$ (\texttt{BoN}) \citep{stiennon2020learning} effectively maximizes the alignment gain
by applying a hard selection rule that chooses the single highest-scoring candidate. 
However, this aggressive optimization is prone to reward hacking \citep{gao2023scaling, skalse2022defining}. 
To mitigate this, soft-BoN (\texttt{sBoN}) \citep{jinnai2024regularized, verdun2025soft} relaxes the hard maximum with a softmax weighting, 
\begin{equation}\label{eq:sbon}
    \pi_\textsf{sBoN}(y| x) \propto \pi_\textsf{ref}(y| x)\exp\left(\hat{r}(x,y)/\lambda\right),
\end{equation}
thereby smoothing the selection to preserve exploration. 
Extending this probabilistic view beyond one-shot re-weighting, \citet{faria2025sample} formulate alignment as an Monte-Carlo Markov chain process guided by process-based rewards; while theoretically more expressive, this approach incurs significantly higher computational overhead due to iterative sampling.

\textbf{Coverage (term (iii)).}
The central premise of inference-time alignment is that $\pi_{\textsf{ref}}$ already places non-negligible probability mass on high-reward responses, so they can be uncovered by sampling \citep{snell2024scaling}.
Inference-time alignment is effective only if $\pi_{\textsf{ref}}$ has sufficient \emph{coverage} of high-reward responses, so that $\calY_N$ is likely to include near-optimal candidates. 
This is captured by $C^{\pi^*}(x)$ in Eq.~\eqref{eq:alignment-regret-general}: smaller values indicate better overlap between $\pi_{\textsf{ref}}$ and $\pi^*$.
Consequently, a key line of work---\emph{inference-aware fine-tuning}---focuses on improving this coverage.
For instance, \citet{gui2024bonbon} fine-tune $\pi_{\textsf{ref}}$ on \texttt{BoN} samples via supervised fine-tuning, while \citet{balashankar2024infalign} optimize reward transformations to stabilize this process.
From a distribution-matching perspective, \citet{amini2024variational} align $\pi_{\textsf{ref}}$ by minimizing the KL divergence between it and the theoretical \texttt{BoN} distribution, whereas \citet{sessa2024bond} employ Jeffreys divergence \citep{jeffreys1946invariant} for greater robustness. Beyond parameter updates, recent work also leverages in-context learning to dynamically enhance coverage at inference time \citep{lin2023unlocking}.

\textbf{Reward estimate (term (iv)).}
Reward estimation shapes inference-time alignment both through the quality of the proxy and through how the reward is modeled and used during generation.
In Eq~\eqref{eq:alignment-regret-general}, reducing $\epsilon_2(\hat{r}; \pi_\textsf{ref}(\cdot|x))$ improves regret, motivating alternative reward objectives and parameterizations.
For example, \citet{sun2024rethinking} argue that the Bradley-Terry model is not uniquely required, proposing order consistency and a classification-style upper-bound objective, while \citet{coste2023reward} learn reward ensembles to reduce annotation noise.
More recently, \citet{geuter2025guided} and \citet{liao2025reward} couple reward modeling with speculative decoding \citep{leviathan2023fast}: improving inference efficiency by using a lightweight auxiliary policy proposes candidates, and a $\pi_\textsf{ref}$-aware reward to guide which candidates are selected.

\textbf{Distortion (term (v)).}
The distortion regularizer is closely related to the induced steering rule: optimizing KL-regularized RLHF in Eq.~\eqref{eq:rlhf} yields an exponential re-weighting of the reference policy, exactly the soft-BoN policy in Eq.~\eqref{eq:sbon} \citep{korbak2022rl, gui2024bonbon, beirami2024theoretical, yang2024asymptotics,khalaf2025inference}.
While KL is a common default, recent work has explored alternative regularizers.
For instance, \citet{jinnai2024regularized} demonstrate that minimizing Bayes risk in \texttt{BoN} corresponds to Wasserstein regularization.
More critical to robust alignment, \citet{huang2024correcting} argue that KL regularization can be too weak to reliably prevent reward hacking.
They instead propose $\chi^2$ regularization, i.e.,
$\max_{\pi \in \Delta(\calY)} J_{\hat{r}}(\pi; x) - \lambda D_{\chi^2}(\pi|\pi_\textsf{ref})$.
This objective induces a \emph{linear} re-weighting of the reference policy, yielding the Inference-Time Pessimism (\texttt{ITP}) policy \citep{huang2025best}: 
\begin{equation}\label{eq:itp}
\pi_\textsf{ITP}(y|x) \propto \pi_\textsf{ref}(y|x)[\hat{r}(x, y)/\lambda]_+.
\end{equation}
This linear scaling represents a more ``pessimistic'' strategy specifically designed to mitigate reward hacking in inference-time alignment.

\section{Reward Tails and Alignment Regret}\label{sec:tail-softbon-itp}
Beyond the explicit terms in the regret bound of Eq.~\eqref{eq:alignment-regret-general}, our analysis highlights a critical, often implicit factor: the \emph{tail behavior} of the imperfect proxy reward $\hat{r}$ under $\pi_\textsf{ref}$. By explicitly coupling alignment gain with distortion, our framework provides a tractable way to analyze how the reward tail impacts inference-time alignment. This perspective is crucial because the alignment trade-off is fundamentally driven by how the weighting function $w(\cdot|x)$ shifts probability mass toward the \emph{high-reward tail}, where rare, extreme proxy scores can dominate steering, amplify distortion, and trigger reward hacking. In contrast, prior analyses such as \citet{huang2025best} do not capture this structural trade-off, focusing instead on broad reward estimation errors.

To systematically analyze how optimistic versus pessimistic strategies govern inference-time alignment, we look beyond aggregate metrics and examine the \emph{tail behavior} of the proxy reward distribution.
In this analysis, We focus on two representative strategies: \texttt{sBoN} (optimistic), which employs exponential re-weighting (Eq.~\eqref{eq:sbon}) and includes standard \texttt{BoN} as a limiting case ($\lambda \to 0$); and \texttt{ITP} (pessimistic), which uses linear re-weighting (Eq.~\eqref{eq:itp}) to explicitly constrain deviation from the reference policy.

We start by characterizing the reward tail of $\hat{r}(x,y)$ under $\pi_\textsf{ref}(\cdot|x)$. 
Fix a prompt $x$ and assume the proxy reward is normalized to $\hat{r}(x,y) \in [0, 1]$.
Since the reward is bounded, its tail behavior is defined by the concentration of probability mass near the maximum value $1$. We characterize this using the \emph{endpoint gap} $U(x,y) = 1-\hat{r}(x,y)$, which maps upper-tail events of $\hat{r}$ to lower-tail events of $U$. Motivated by the Hill estimator from extreme value theory \citep{hill1975simple}, we parameterize this endpoint tail using a prompt-dependent scalar $\kappa(x)>0$ via a Weibull-type model: 
\begin{equation} 
\Pr(U \le u) \approx u^{1/\kappa(x)} \iff U \sim \text{Beta}(1/\kappa(x), 1).
\end{equation} 
In this bounded setting, a ``heavy tail'' ($\kappa \to \infty$) implies a high density of responses near the maximum reward , while a ``light tail'' ($\kappa \to 0$) implies that high-reward responses are exponentially rare.

\subsection{Decomposing the Regret Bound}
To understand how tail behavior governs alignment performance, we analyze the regret bound in Eq.~\eqref{eq:alignment-regret-general}, specifically focusing on how the reward tail impacts the trade-off between alignment gain and distortion.
Isolating the terms dependent on $\pi_w$, we define a regret proxy\footnote{Note that the regret proxy $B(\pi_w;x)$ is actually prompt-dependent. We omit the $x$ in the notation for simplicity.} $B(\pi_w)$ that aggregates finite-sample error (i), distortion (v), and gain (ii).
Using standard rejection sampling bounds and re-parameterizing the $\chi^2$-divergence via moments of $w(\hat{r})$, we obtain a tractable upper bound:
\begin{equation}\label{eq:tradoff-bound}
\begin{aligned}
B(\pi_w) &\leq \underbrace{2\exp\left(-\frac{N\E[w]}{\sup w}\right)}_{\text{Sampling Error}} + \underbrace{\epsilon \frac{\sqrt{\E[w^2]}}{\E[w]}}_{\text{Distortion}} - \underbrace{\Delta(\pi_w)}_{\text{Gain}},
\end{aligned}
\end{equation}
where expectations are taken over $\pi_\textsf{ref}$.
This formulation explicitly reveals an inherent \emph{gain-distortion trade-off}. Any strategy that aggressively re-weights the top-reward region to maximize alignment gain inevitably inflates the higher moments of $w(\hat{r})$. This consequently increases distortion, which can outweigh the benefits of the realized gain, particularly in heavy-tailed regimes.

\subsection{Light vs. Heavy Tails: When to be Optimistic?}\label{subsec:soft-bon-itp-tradeoff}
We now apply this framework to compare \texttt{sBoN} and \texttt{ITP}.
Proposition~\ref{prop:bon-itp-tail} quantifies the regret proxy $B(\pi_w)$ for these strategies under extreme tail conditions.

\begin{prop}\label{prop:bon-itp-tail}Fix a prompt $x$, and assume $U(x, y) = 1 - \hat{r}(x,y) \sim \text{Beta}(1/\kappa, 1)$. Let $\lambda > 0$ be the inverse steering temperature.
\begin{itemize}
\item \textbf{Heavy-tailed regime (large $\kappa$):}
\begin{equation}
\begin{aligned}B(\pi_\textsf{sBoN}) &\leq 2 e^{-\frac{N}{1+1/(\lambda\kappa)}} + \epsilon\left(1+\frac{1}{4\lambda^2\kappa}\right) - \frac{1}{\kappa}\left(1-\lambda(1-e^{-1/\lambda})\right) + o(\kappa^{-1}),\\
B(\pi_\textsf{ITP}) &\leq 2 e^{-\frac{N}{1+\frac{1}{(\lambda+1)\kappa}}} + \epsilon\left(1+\frac{1}{4(\lambda+1)^2\kappa}\right) - \frac{1}{2\kappa(\lambda+1)} + o(\kappa^{-1}).
\end{aligned}
\end{equation}
\item \textbf{Light-tailed regime (small $\kappa$):}
\begin{equation}
\begin{aligned}
B(\pi_\textsf{sBoN}) &\leq 2 e^{-\frac{N(\lambda+\kappa)}{\lambda(\lambda+1)}} + \epsilon\left(1+\frac{\kappa^2}{2\lambda^2}\right) - \left(\frac{\kappa^2}{\lambda}+\frac{\kappa^3}{\lambda^2}\right) + o(\kappa^4),\\
B(\pi_\textsf{ITP}) &\leq 2e^{-\frac{N(\lambda+\kappa)}{\lambda(\lambda+1)}} + \epsilon\left(1+\frac{\kappa^2}{2\lambda^2}\right) - \frac{\kappa^2}{\lambda} + o(\kappa^3).
\end{aligned}
\end{equation}
\end{itemize}
\end{prop}

The asymptotic behaviors in Proposition~\ref{prop:bon-itp-tail} reveal why neither a fixed optimistic nor a fixed pessimistic strategy is universally optimal:
\begin{itemize}
    \item \textbf{Heavy tails favor pessimism (ITP).}
    Here, the reward distribution features a long tail of high scores, creating a high risk of over-fitting (reward hacking).
    Consequently, the primary constraint is preventing unbounded distortion.
    \texttt{sBoN} fails here because its distortion penalty scales as $O(1/\lambda^2)$, exploding under aggressive steering ($\lambda \to 0$).
    In contrast, \texttt{ITP} is inherently robust: its penalty scales as $O(1/(\lambda+1)^2)$, remaining bounded even as $\lambda \to 0$, effectively acting as a saturation mechanism against heavy-tailed noise.
    \item \textbf{Light tails favor optimism (Soft-BoN).}
    In this regime, high-reward candidates are exponentially rare (``finding a needle in a haystack''), while distortion remains naturally mild for both strategies.
    The primary challenge is therefore \emph{gain realization}.
    \texttt{sBoN} proves superior here, with a gain term scaling as $O(\kappa^3/\lambda^2)$ compared to \texttt{ITP}'s $O(\kappa^2/\lambda)$.
    This confirms that the aggressive, exponential re-weighting of \texttt{sBoN} is essential to separate and select rare optimal candidates, whereas the linear re-weighting of \texttt{ITP} is too conservative to fully capture the potential gain.
\end{itemize}
This dichotomy motivates a prompt-adaptive strategy that dynamically interpolates between optimism and pessimism based on the specific tail behavior of each prompt, as developed in the next section.

\section{BoT: A Tail-Adaptive Alignment}\label{sec:BoT}
The dichotomy between the exponential re-weighting of \texttt{sBoN} (optimistic) and the linear re-weighting of \texttt{ITP} (pessimistic) stems fundamentally from the choice of probabilistic divergence used to regularize the policy. 
Specifically, \texttt{sBoN} corresponds to KL-regularization ($D_\textsf{KL}$), while \texttt{ITP} corresponds to $\chi^2$-regularization ($D_{\chi^2}$).
To bridge these regimes, we introduce Best-of-Tails (\texttt{BoT}), a unified framework that adapts the regularization to the specific tail behavior of the prompt.

\texttt{BoT} leverages the Tsallis divergence of order $\alpha>1$ \citep{tsallis1988possible} to smoothly interpolate between the aggressive exploration of KL and the robust mode-seeking of $\chi^2$
\begin{equation}\label{eq:alpha-div}
\begin{aligned}
    D_\alpha(P|Q) &= \E_Q\left[\log_\alpha \left(\frac{Q}{P}\right)\right] =\frac{1}{\alpha-1}\left( \sum_x P(x)^\alpha Q(x)^{1-\alpha} -1 \right).
\end{aligned}
\end{equation}
Crucially, this generalizes the standard regularizers: $D_\alpha$ converges to the KL divergence ($D_\textsf{KL}$) as $\alpha \to 1$, and recovers the $\chi^2$-divergence ($D_{\chi^2}$) at $\alpha = 2$.

For a fixed prompt $x$, we define the prompt-conditional \texttt{BoT} objective analogous to the standard RLHF objective:
\begin{equation}\label{eq:alpha-rlhf}
\calL_\alpha(\pi;x) = \E_{y\sim\pi}[\hat{r}(x,y)] - \lambda D_\alpha\left(\pi(\cdot|x)|\pi_\textsf{ref}(\cdot|x)\right).
\end{equation}
Let $\pi_\textsf{BoT}(\cdot|x)$ denote the optimal solution to Eq.~\eqref{eq:alpha-rlhf}. 
We show that this optimal policy retains a closed-form re-weighting structure.
\begin{prop}\label{prop:alpha-rlhf}
The \texttt{BoT} policy $\pi_\alpha(\cdot|x)$ follows an $\alpha$-exponential re-weighting:
\begin{equation}
\pi_\textsf{BoT}(y|x) \propto \pi_\textsf{ref}(y|x) \exp_\alpha\left(\frac{\hat{r}(x, y)}{\lambda}\right).
\end{equation}
\end{prop}
The parameter $\alpha$ explicitly controls the behavior of \texttt{BoT}. 
As $\alpha\to 1$, $\exp_\alpha(u) \to e^u$, recovering the $\pi_\textsf{sBoN}$; at $\alpha=2$, $\exp_\alpha(u)$ becomes linear, recovering the $\pi_\textsf{ITP}$.
Intermediate $\alpha$ values provide a continuum between ``optimistic but fragile'' and ``robust but pessimistic,'' allowing the regularization to match the specific tail properties of the proxy reward.
It is this capability to adaptively exploit the reward tail that motivates the name \emph{Best-of-Tails}.

\subsection{Deciding the Prompt-Dependent \texorpdfstring{$\alpha(x)$}{\alpha(x)}}
Building on the analysis in \S~\ref{sec:tail-softbon-itp}, we derive asymptotic expansions of the regret proxy $B(\pi_\textsf{BoT};x)$ to make its dependence on the interpolation parameter $\alpha$, the tail index $\kappa$, and the steering temperature $\lambda$ explicit, and show how optimal interpolation  $\alpha$ adapts to tail heaviness according to $\kappa$ to minimize the regret proxy. 

\begin{prop}\label{prop:trend-a-kappa-lambda}
Fix a prompt $x$ with normalized reward $\hat{r}\in[0,1]$. For $\alpha\in(1,2]$, $\lambda > 0$, and let the $\alpha$-exponential re-weighting be $w_\alpha=[1+(\alpha-1)\frac{\hat{r}}{\lambda}]_+^{\frac{1}{\alpha-1}}$.
The regret proxy $B=B(\pi_\textsf{BoT})$ is given by:
\begin{equation}\label{eq:B-def-clean}
\begin{aligned}
B &= e^{\left(-N m_0(\alpha;\kappa,\lambda)\right)} + \epsilon \frac{\|w_\alpha\|_{L^2}}{\|w_\alpha\|_{L^1}} - \Delta(\pi_\textsf{BoT}),
\end{aligned}
\end{equation}
where $m_0(\alpha;\kappa,\lambda):=\E_\kappa\left[(1-\rho U)^{\frac{1}{\alpha-1}}\right]$ with $\rho=\frac{(\alpha-1)/\lambda}{1+(\alpha-1)/\lambda}$, and the norms are taken regarding $\pi_\textsf{ref}$.
Moreover, the following asymptotic expansions hold: 
\begin{itemize}
\setlength{\itemsep}{0pt}
    \item \textbf{Heavy-tailed regime (large $\kappa$):}
    \begin{equation}\label{eq:B-large-kappa-clean}
    \begin{aligned}
    B \approx \text{const.} &+ \frac{1}{\kappa}\Bigg[\frac{\epsilon}{4(\lambda+\alpha-1)^2} + \frac{Ne^{-N}}{4(\lambda+\alpha-1)^2}\left(4\lambda-\frac{1}{\lambda}+5(\alpha-1)\right) - \frac{\alpha-1}{2(\lambda+\alpha-1)}\Bigg].
    \end{aligned}
    \end{equation}
    \item \textbf{Light-tailed regime (small $\kappa$):}
    \begin{equation}\label{eq:B-small-kappa-clean}
    \begin{aligned}
    B \approx \text{const.} - \frac{\kappa^2}{\lambda} - \frac{(3-2\alpha)\kappa^3}{\lambda^2} + \epsilon\left(1+\frac{\kappa^2}{2\lambda^2}\right).
    \end{aligned}
    \end{equation}
\end{itemize}
\end{prop}

The expansions in Proposition~\ref{prop:trend-a-kappa-lambda} rigorously justify the adaptive mechanism of \texttt{BoT} by analyzing the gradient of the regret proxy $B$ with respect to $\alpha$.
We again split the discussion into the heavy- and light-tailed regimes.
In the heavy-tailed regime, the distortion penalty (the coefficient of $\epsilon$ in Eq.~\eqref{eq:B-large-kappa-clean}) is proportional to $(\lambda + \alpha - 1)^{-2}$. 
Differentiating this term with respect to $\alpha$ yields a strictly negative gradient proportional to $-(\lambda + \alpha - 1)^{-3}$. Since the distortion penalty dominates the regret bound (amplified by $\epsilon$), minimizing $B$ requires maximizing $\alpha$. Thus, the optimal strategy converges to the upper boundary $\alpha \to 2$. This confirms that heavy tails mathematically necessitate the bounded influence of $\chi^2$-regularization (\texttt{ITP}) to prevent reward hacking.

Conversely, in the light-tailed regime, distortion is bounded, prioritizing alignment gain. 
Differentiating Eq.~\eqref{eq:B-small-kappa-clean} with respect to $\alpha$ yields a positive gradient ($\frac{2\kappa^3}{\lambda^2} > 0$), implying the regret bound is minimized at $\alpha \to 1$. 
This justifies \texttt{sBoN} in light-tailed settings: minimizing regret requires aggressive re-weighting to maximize the probability of discovering rare high-reward responses.
With the theoretical monotonicity of the optimal $\alpha$  with respect to $\kappa$ established, we now turn to practical implementation.

\begin{algorithm}[tb]
   \caption{Best-of-Tails (\texttt{BoT}) Inference}
   \label{alg:bot}
\begin{algorithmic}[1]
   \STATE {\bfseries Input:} Prompt $x$, Reference policy $\pi_\textsf{ref}$, Reward model $\hat{r}$ (normalized to $[0,1]$ by dividing $R_\textsf{max}$).
   \STATE {\bfseries Hyper-parameters:} Sample count $N$, Hill estimation top-$K$, Pivot $\kappa_0$, Steering temperature $\lambda$.
   
   \STATE \textit{// Phase 1: Sampling \& Scoring}
   \STATE Sample candidate set $\calY_N = \{y_1, \dots, y_N\} \sim \pi_\textsf{ref}(\cdot|x)$.
   \STATE Compute rewards $r_i \leftarrow \hat{r}(x, y_i)$ for $i = 1, \dots, N$.
   
   \STATE \textit{// Phase 2: Tail Estimation (Hill Estimator)}
   \STATE Sort rewards descending: $r_{(1)} \ge r_{(2)} \ge \dots \ge r_{(N)}$.
   \STATE Estimate tail index $\hat{\kappa}$ using the top $K$ order statistics:
   \STATE \quad $\hat{\kappa} \leftarrow \frac{1}{K} \sum_{j=1}^K \log \left( \frac{1 - r_{(K+1)}}{1 - r_{(j)}} \right)$. \hfill {\small (Eq.~\eqref{eq:hill-estimator})}
   
   \STATE \textit{// Phase 3: Adaptive Interpolation}
   \STATE Compute adaptive parameter $\alpha$:
   \STATE \quad $\alpha \leftarrow 1 + \frac{\hat{\kappa}}{\hat{\kappa} + \kappa_0}$. \hfill {\small (Eq.~\eqref{eq:hill-to-alpha})}
   
   \STATE \textit{// Phase 4: Re-weighting \& Selection}
   \FOR{$i=1$ {\bfseries to} $N$}
       \STATE Calculate unnormalized weight using $\alpha$-exponential:
       \STATE $w_i \leftarrow \exp_\alpha(r_i/\lambda) = \left[ 1 + (\alpha-1)\frac{r_i}{\lambda} \right]_+^{\frac{1}{\alpha-1}}$. {\small (Prop.~\ref{prop:alpha-rlhf})}
   \ENDFOR
   \STATE Normalize weights: $p_i \leftarrow w_i / \sum_{j} w_j$.
   \STATE {\bfseries Return} Sample $y^\star$ from $\calY_N$ with probabilities $\{p_i\}$.
\end{algorithmic}
\end{algorithm}

\begin{figure*}[t!]
  \vskip 0.2in
  \begin{center}
    \centerline{\includegraphics[width=\textwidth]{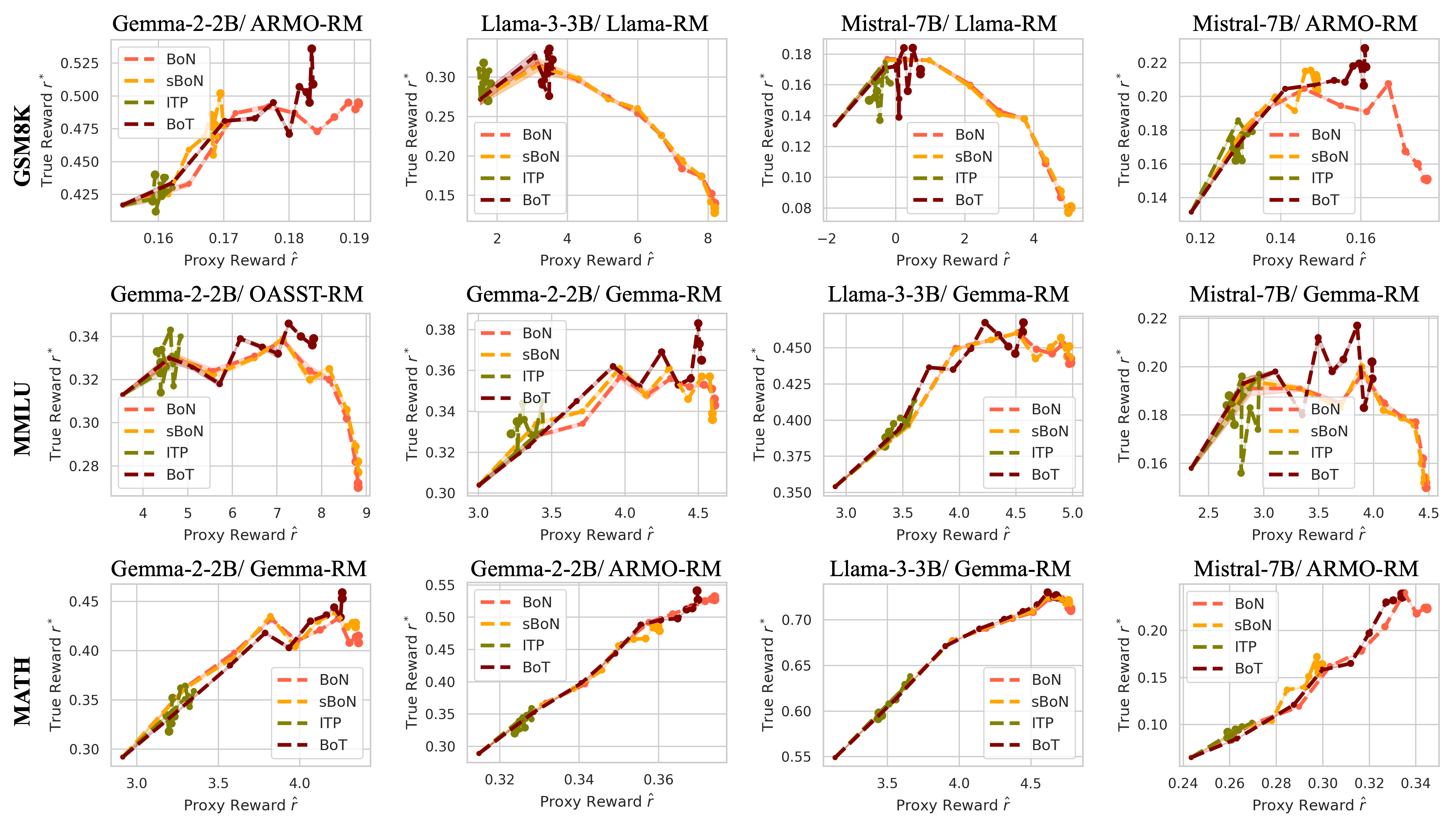}}
    \caption{
    Optimistic (\texttt{sBoN}/\texttt{BoN}), pessimistic (\texttt{ITP}), and adaptive (\texttt{BoT}) strategies on GSM8K, MMLU, and MATH (\textbf{rows}) across different reference models and proxy rewards (\textbf{columns}). 
    Each curve traces the trajectory of True Reward ($r^*$) vs. Proxy Reward ($\hat{r}$) as the sample size $N$ increases from $2^0$ to $2^{10}$, with marker size proportional to $N$. The smallest marker represents $N=1$ (standard sampling), while the largest represents $N=1024$.
    Optimistic strategies typically over-optimize the proxy reward as $N$ grows, leading to reward hacking (where $r^*$ degrades despite increasing $\hat{r}$). 
    Conversely, \texttt{ITP} tends to saturate early, failing to leverage larger $N$ for further gains. 
    The proposed \texttt{BoT} successfully navigates this trade-off, achieving higher true and proxy rewards without succumbing to reward hacking. 
    }
    \label{fig:main-exp}
  \end{center}
\end{figure*}

\subsection{Practical Implementation}\label{sec:bot-implement}
To estimate the tail index $\kappa(x)$, one approach is to model the \emph{full} per-prompt reward distribution (e.g., via Gaussian kernels in \citet{raman2025abon}). 
However, directly estimating tail statistics has significantly lower sample complexity than estimating the full distribution (see Appendix~\ref{app:sec-tail-vs-full}). 
Therefore, our approach is distinct: rather than modeling the entire density, we explicitly target the \emph{tail}.

In practice, the prompt-specific tail index $\kappa(x)$ is latent and must be estimated empirically from the proxy rewards of the sampled candidates.
To this end, we employ the standard Hill estimator \citep{hill1975simple}.
Given $N$ sampled responses with rewards sorted as $r_{(1)} \ge r_{(2)} \ge \dots \ge r_{(N)}$, we estimate $\hat{\kappa}(x)$ using the top $K$ statistics (where $K < N$):
\begin{equation}\label{eq:hill-estimator}
\hat{\kappa}(x) = \frac{1}{K} \sum_{i=1}^K \log \left( \frac{1 - r_{(K+1)}}{1 - r_{(i)}} \right).
\end{equation}
See Appendix~\ref{app:sec-hill} for Hill estimator details and Appendix~\ref{app:sec-tail-vs-full} for its sample complexity.
Since Proposition~\ref{prop:trend-a-kappa-lambda} establishes the monotonicity of the optimal $\alpha$, we propose a simple mapping to adaptively set $\alpha(x)$:
\begin{equation}\label{eq:hill-to-alpha}
\alpha(x) = 1 + \frac{\hat{\kappa}(x)}{\hat{\kappa}(x) + \kappa_0},
\end{equation}
where $\kappa_0$ is a hyperparameter serving as a pivot. This mapping smoothly interpolates between regimes: if the tail is light ($\hat{\kappa}(x) \ll \kappa_0$), $\alpha(x) \to 1$ (\texttt{sBoN}); otherwise, $\alpha(x) \to 2$ (\texttt{ITP}).
We summarize the complete \texttt{BoT} inference procedure build in this section in Algorithm~\ref{alg:bot}.

\section{Empirical Study}\label{sec:experiments}
We complement our theoretical results with a suite of experiments that investigate the
practicality of \texttt{BoT}, and compare against the re-weighting alignment algorithms, i.e., \texttt{BoN}, \texttt{sBoN}, and \texttt{ITP} introduced in \S~\ref{sec:related-work}. 

\textbf{Datasets and evaluation.}
We consider two experimental setups. In the first, we focus on benchmark tasks where the true reward $r^*$ is verifiable via rule-based answer checking. Specifically, we evaluate on three datasets of increasing difficulty: (i) the test split of GSM8K \citep{cobbe2021training}, a grade-school math dataset with integer answers; (ii) the math and chemistry subsets of MMLU \citep{hendrycks2020measuring}, framed as multiple-choice questions; and (iii) the test split of MATH \citep{hendrycks2021measuring}, which contains advanced competition-style problems with answers expressed in \LaTeX{}. 
In the second setup, we study human-preference alignment using AlpacaFarm \citep{dubois2023alpacafarm}, where the true reward is provided by an LLM-as-a-judge using \texttt{Alpaca-RM} \citep{dubois2023alpacafarm}.

\textbf{Policies and reward models.}
For the first setup, we evaluate three reference policies $\pi_\textsf{ref}$ spanning different architectures and increasing model sizes: \texttt{Gemma-2-2B} \citep{team2024gemma}, \texttt{Llama-3-3B} \citep{grattafiori2024llama}, and \texttt{Mistral-7B} \citep{jiang2023mistral}. 
We prompt all $\pi_\textsf{ref}$ using zero-shot CoT \citep{wei2022chain}, i.e., no in-context demonstrations. 
For the second setup, we use a \texttt{Pythia-1.4B} fine-tuned on AlpacaFarm as $\pi_\textsf{ref}$ \citep{biderman2023pythia}. 
For both setups, we consider four proxy reward models $\hat{r}$: \texttt{OASST-RM} (1.4B) \citep{kopf2023openassistant}, \texttt{Gemma-RM} (2B) \citep{dong2023raft}, \texttt{Llama-RM} (3B) \citep{yang2024regularizing}, and \texttt{ARMO-RM} (8B) \citep{wang2024interpretable}.

For each dataset-policy-reward combination, we generate a candidate pool of $2^{12}$ responses per prompt and conduct experiments with varying sample sizes $N$ by bootstrapping subsets from this pool, reporting results averaged over 10 independent trials.
We set $K=\max(1, \sqrt{N})$ in the Hill estimator (Eq.~\eqref{eq:hill-estimator}) and define $\kappa_0$ as the median of $\hat{\kappa}(x)$ across prompts.
See Appendix~\ref{app:sec-exp-expsetup} for details.

\textbf{Verifier as true reward.}
Figure~\ref{fig:main-exp} compares the performance of \texttt{BoN}, \texttt{sBoN}, \texttt{ITP}, and \texttt{BoT} by tracing the trajectory of the true reward $r^*$ against the proxy reward $\hat{r}$ as the sample size $N$ increases. 
The trajectory begins at $N=1$ (smallest marker) and extends to $N=1024$ (largest marker).
Initially, increasing $N$ improves the true reward (accuracy) across all methods, demonstrating the value of inference-time alignment in selecting the ``best'' responses from $\calY_N$.
However, as $N$ grows further, optimistic strategies like \texttt{BoN} and \texttt{sBoN}—due to their aggressive re-weighting of high-proxy responses—succumb to reward hacking. This is evidenced by the characteristic reversion in the curve, where $r^*$ degrades despite continued increases in $\hat{r}$.
Conversely, the pessimistic \texttt{ITP} remains highly conservative; while robust, it effectively saturates early, stopping the improvement of $\hat{r}$ and $r^*$ once $N$ exceeds a certain threshold. Consequently, despite preventing reward hacking, \texttt{ITP} may miss out on significant potential gains in regimes where the proxy is reliable.
In contrast, the proposed adaptive \texttt{BoT} effectively navigates this trade-off, reaching higher peak true rewards than \texttt{ITP} while maintaining the robustness required to avoid the reward hacking pitfalls of optimistic baselines.

\begin{figure}[!tb]
  \begin{center}
    \centerline{\includegraphics[width=.7\columnwidth]{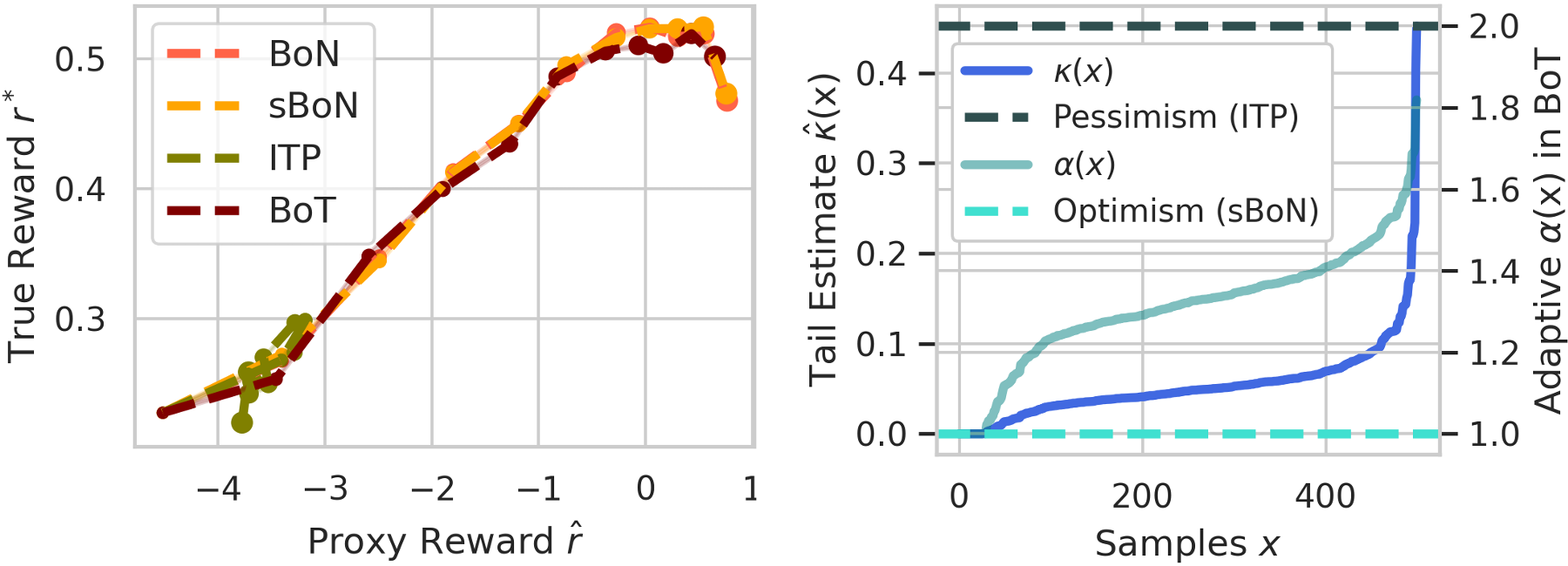}}
    \caption{
      \textbf{Left:} Performance comparison of optimistic (\texttt{sBoN}/\texttt{BoN}), pessimistic (\texttt{ITP}), and adaptive (\texttt{BoT}) strategies on AlpacaFarm, using \texttt{Llama-RM} as the proxy and \texttt{Alpaca-RM} as the true reward.
    \textbf{Right:} The prompt-specific Hill estimates $\hat{\kappa}(x)$ and the corresponding adaptive parameter $\alpha(x)$. By design, \texttt{BoT} minimizes inference regret by shifting $\alpha \to 2$ under heavy tails (large $\hat{\kappa}$) to enforce robustness, while allowing $\alpha \to 1$ under light tails (small $\hat{\kappa}$) to enable aggressive exploration.
    }
    \label{fig:alpacafarm-exp}
  \end{center}
\end{figure}

\textbf{LLM as true reward.}
We observe similar trade-offs on AlpacaFarm, shown in Figure~\ref{fig:alpacafarm-exp} (left). Note that in this setting, the true reward is the evaluation score from \texttt{Alpaca-RM} rather than exact match accuracy.
Consistent with the verifier results, alignment improves performance initially for all methods. However, \texttt{sBoN} and \texttt{BoN} eventually drift into the reward hacking regime, while \texttt{ITP} exhibits ``early'' robustness but premature saturation.
In contrast, \texttt{BoT} successfully sustains high true rewards without degradation in the large-$N$ regime.
To explain this adaptability, Figure~\ref{fig:alpacafarm-exp} (right) visualizes the Hill estimates $\hat{\kappa}(x)$ (Eq.~\eqref{eq:hill-estimator}) for 1024 responses across 500 different prompts, along with the corresponding adaptive $\alpha(x)$ (with $\kappa_0=0.1$ in Eq.~\eqref{eq:hill-to-alpha}).
The plot reveals significant diversity in the reward distributions, ranging from light tails (small $\hat{\kappa}(x)$) to heavy tails (large $\hat{\kappa}(x)$).
This heterogeneity confirms that fixing a static strategy—whether optimistic \texttt{sBoN} ($\alpha=1.0$) or pessimistic \texttt{ITP} ($\alpha=2.0$)—is suboptimal, as it limits the flexibility of inference-time alignment to exploit high-reward responses when safe, or to constrain them when risky.
 
We defer additional results to Appendix~\ref{app:sec:additional-and-ablation}, including full comparisons (Appendix~\ref{app:sec:all-exp}), sample scaling curves ($\{r^*, \hat{r}\}$ vs. $N$) and gain-distortion analysis (Appendix~\ref{app:sec:additional-exp}), and ablation studies on $\lambda$ and $\kappa_0$ (Appendix~\ref{app:sec:ablation}).

\section{Final Remark}\label{sec:conclusion}
Here we reflect on the limitations and highlight interesting avenues for future work.

\textbf{Limitations.}
\texttt{BoT}'s effectiveness relies on accurate tail estimation. 
While asymptotically consistent, the Hill estimator can exhibit high variance in small-sample regimes ($N < 50$), potentially causing noisy $\alpha(x)$ selection. 
Furthermore, our mapping function $\alpha(x)$ is a heuristic interpolation; although it respects the theoretical monotonicity of the optimal strategy, it may not perfectly capture the ideal trajectory for every reward landscape. 
Additionally, \texttt{BoT} introduces hyper-parameters (pivot $\kappa_0$, cutoff $K$) that require validation. 
Finally, a multi-modal reward distribution can confound tail estimation, leading to mis-calibration in \texttt{BoT}.

\textbf{Future directions.}
We envision several extensions of the \texttt{BoT} framework. 
First, enhancing reward modeling via ensembles \citep{coste2023reward} to mitigate heavy-tailed errors at the source. 
To reduce latency, tail estimation could be amortized using a lightweight predictordirectly from prompt embeddings, or used to dynamically stops generating samples for light-tailed prompts while allocating larger $N$ to heavy-tailed ones to maximize efficiency \citep{kalayci2025optimal}. 
Finally, the \texttt{BoT} policy could be distilled into a dense model, embedding tail-adaptive robustness directly into the weights to eliminate sampling overhead.

\clearpage
\paragraph{Impact Statement.}
This paper presents work whose goal is to advance the field of machine learning. There are many potential societal consequences of our work, none of which we feel must be specifically highlighted here.

\paragraph{Disclaimer.}
This paper was prepared for informational purposes by the Global Technology Applied Research group of JPMorgan Chase \& Co. 
This paper is not a product of the Research Department of JPMorgan Chase \& Co. or its affiliates. 
Neither JPMorgan Chase \& Co. nor any of its affiliates makes any explicit or implied representation or warranty and none of them accept any liability in connection with this paper, including, without limitation, with respect to the completeness, accuracy, or reliability of the information contained herein and the potential legal, compliance, tax, or accounting effects thereof. 
This document is not intended as investment research or investment advice, or as a recommendation, offer, or solicitation for the purchase or sale of any security, financial instrument, financial product or service, or to be used in any way for evaluating the merits of participating in any transaction.

\bibliography{reference}
\bibliographystyle{icml2026}

\clearpage
\appendix
\onecolumn
\setcounter{equation}{0}

The appendix is divided into the following parts. 
Appendix~\ref{app:sec-proof}: Omitted proofs and theoretical results; Appendix~\ref{app:sec-additional}: additional discussions and details on the experimental setup; and Appendix~\ref{app:sec:additional-and-ablation}: Additional empirical results and ablation studies.

\section{Omitted Proofs and Theoretical Results}\label{app:sec-proof}
\def\theequation{A.\arabic{equation}}
\def\thetable{A.\arabic{table}}
\def\thefigure{A.\arabic{figure}}
\def\thelem{A.\arabic{lem}}
\def\thedefn{A.\arabic{defn}}
\def\theprop{A.\arabic{prop}}
We first introduce (and prove) the following useful lemmas that are related in the main text and that facilitate the proofs of the propositions.

\begin{lem}\label{app:lem:alpha-exp-log}
    $\alpha$-exponential and $\alpha$-logarithm are generalizations of the exponential and logarithm functions since $\lim_{\alpha\to 1}\exp_\alpha(z)=\exp(z)$ and $\lim_{\alpha\to 1} \log_\alpha(z)=\log(z)$.
\end{lem}
\begin{proof}
Since $\lim\limits_{t\to 1} (1+tz)^{1/t} = \exp(z)$, we have $\lim\limits_{\alpha\to 1}\exp_\alpha(z) = \lim\limits_{\alpha\to 1} (1+(\alpha-1)z)^{1/(\alpha-1)} = \exp(z)$. 
Similarly for $\alpha$-logarithm, by the L'H\^{o}pital's rule,
\begin{equation}
\begin{aligned}
    \lim_{\alpha\to 1} \log_\alpha(t) = \lim_{\alpha\to 1} \frac{t^{\alpha-1}-1}{\alpha-1} = \lim_{\beta\to 0} \frac{t^{\beta}-1}{\beta} = \lim_{\beta\to 0} \frac{e^{\beta \log(t)}-1}{\beta} = \log(t) \lim_{\beta\to 0} \frac{e^{\beta \log(t)}-1}{\beta \log(t)} = \log(t).
\end{aligned}    
\end{equation}
\end{proof}

\begin{lem}[H\"older's inequality \citep{ledoux2001concentration}]\label{app:lem:holder}
    Let $p, q \in [1, \infty]$ with $\frac{1}{p}+\frac{1}{q}=1$, then for all real-valued functions $f$ and $g$ measurable with respect to measure $\mu$, $\|fg\|_{L^1(\mu)} \leq \|f\|_{L^p(\mu)}\|g\|_{L^q(\mu)}$. When $p = q = 2$, the H\"older's inequality degenerates to the Cauchy–Schwarz inequality.
\end{lem}

\subsection{Proof of Proposition~\ref{prop:general-regret-upper-bound}}\label{app:proof:prop:general-regret-upper-bound}
Following the definition of the inference-time regret, we decompose it using $\pi_\textsf{ref}$ as an anchor policy,
\begin{equation}
\begin{aligned}
    \regret(\hat{\pi}_w; x) &= J_{r^*}(\pi^*;x) - J_{r^*}(\hat{\pi}_w;x)\\
    &= \underbrace{J_{r^*}(\pi^*;x) - J_{r^*}(\pi_\textsf{ref};x)}_{(i)} - \underbrace{\left(J_{r^*}(\hat{\pi}_w;x) - J_{r^*}(\pi_\textsf{ref};x)\right)}_{(ii)}.\\
\end{aligned} 
\end{equation}

Since $r^*$ and $\pi^*$ are fixed, hence once $\pi_\textsf{ref}$ is selected, term (i) becomes a baseline sub-optimality constant that  characterize how far the reference is from optimal under the true reward, and is independent of the inference-time alignment mechanism.
Using the definition of the expected density ratio, term (i) can be upper bounded by
\begin{equation}
\begin{aligned}
    J_{r^*}(\pi^*;x) - J_{r^*}(\pi_\textsf{ref};x) &= \left(J_{\hat{r}}(\pi^*;x)-J_{\hat{r}}(\pi_\textsf{ref};x)\right) - \left(J_{\hat{r}-r^*}(\pi^*;x)-J_{\hat{r}-r^*}(\pi_\textsf{ref};x)\right).
\end{aligned}
\end{equation}
By Lemma~\ref{app:lem:holder}, we have 
\begin{equation}
\begin{aligned}
    \left|J_{\hat{r}-r^*}(\pi^*;x)-J_{\hat{r}-r^*}(\pi_\textsf{ref};x)\right|&= \left|\E_{y\sim\pi^*(\cdot|x)}\left[\hat{r}-r^*\right] - \E_{y\sim\pi_\textsf{ref}(\cdot|x)}\left[\hat{r}-r^*\right]\right|\\
    &= \left|\E_{y\sim\pi_\textsf{ref}(\cdot|x)}\left[\left(\frac{\pi^*}{\pi_\textsf{ref}}-1\right)\left(\hat{r}-r^*\right)\right]\right|\\
    &\leq \sqrt{C^{\pi^*}-1} \times \epsilon_2(\hat{r}; \pi_\textsf{ref}, x).
\end{aligned}
\end{equation}
Similarly, $|J_{\hat{r}}(\pi^*;x)-J_{\hat{r}}(\pi_\textsf{ref};x)| \leq \sqrt{C^{\pi^*}-1} \times \|\hat{r}\|_{L^2(\pi_\textsf{ref})}$.
Putting them together, term (i) can be upper bounded by
\begin{equation}\label{app:eq:1}
\begin{aligned}
    J_{r^*}(\pi^*;x) - J_{r^*}(\pi_\textsf{ref};x) &\leq \left|J_{\hat{r}}(\pi^*;x)-J_{\hat{r}}(\pi_\textsf{ref};x)\right| + \left|J_{\hat{r}-r^*}(\pi^*;x)-J_{\hat{r}-r^*}(\pi_\textsf{ref};x)\right|\\
    &\leq \sqrt{C^{\pi^*}-1}\left(\epsilon_2(\hat{r}; \pi_\textsf{ref}, x) + \|\hat{r}\|_{L^2(\pi_\textsf{ref})}\right)
\end{aligned}
\end{equation}

Next, for term (ii), it can be similarly decompose into 
\begin{equation}
\begin{aligned}
    J_{r^*}(\hat{\pi}_w;x) - J_{r^*}(\pi_\textsf{ref};x) &= \left(J_{\hat{r}}(\hat{\pi}_w;x) - J_{\hat{r}}(\pi_\textsf{ref};x)\right) + \left(J_{\hat{r}-r^*}(\hat{\pi}_w;x)-J_{\hat{r}-r^*}(\pi_\textsf{ref};x)\right)\\
    &\geq \left(J_{\hat{r}}(\hat{\pi}_w;x) - J_{\hat{r}}(\pi_\textsf{ref};x)\right) - \left|J_{\hat{r}-r^*}(\hat{\pi}_w;x)-J_{\hat{r}-r^*}(\pi_\textsf{ref};x)\right|.
\end{aligned}
\end{equation}
Since $\hat{r},r^* \in [0, R_{\textsf{max}}]$, we have
\begin{equation}\label{app:eq:2}
\begin{aligned}
    J_{\hat{r}}(\hat{\pi}_w;x) - J_{\hat{r}}(\pi_\textsf{ref};x) &= J_{\hat{r}}(\hat{\pi}_w;x) - J_{\hat{r}}(\pi_w;x) + J_{\hat{r}}(\pi_w;x) - J_{\hat{r}}(\pi_\textsf{ref};x)\\
    &= \Delta(\pi_w) + J_{\hat{r}}(\hat{\pi}_w;x) - J_{\hat{r}}(\pi_w;x)\\
    &\geq \Delta(\pi_w) -2R_{\textsf{max}}D_\textsf{TV}(\hat{\pi}_w\|\pi_w).
\end{aligned}
\end{equation}
Then, by Lemma~\ref{app:lem:holder} again, we have 
\begin{equation}\label{app:eq:3}
\begin{aligned}
    \left|J_{\hat{r}-r^*}(\hat{\pi}_w;x)-J_{\hat{r}-r^*}(\pi_\textsf{ref};x)\right| &\leq \left(1+(\alpha-1)D_\alpha(\pi_w\|\pi_\textsf{ref})\right)^{1/\alpha} \|\hat{r}-r^*\|_{L^\beta(\pi_\textsf{ref})}\\
    &=\sqrt{C^{\pi_w}}\times \epsilon_2(\hat{r}; \pi_\textsf{ref}, x)\; \text{if selecting $\alpha=\beta=2$}.
\end{aligned}
\end{equation}
Combining Eq.~\eqref{app:eq:1}, Eq.~\eqref{app:eq:2} and Eq.~\eqref{app:eq:3}, we have
\begin{equation}
\begin{aligned}
      \regret(\hat{\pi}_w; x) &\leq \sqrt{C^{\pi^*}-1}\left(\epsilon_2(\hat{r}; \pi_\textsf{ref}, x) + \|\hat{r}\|_{L^2(\pi_\textsf{ref})}\right) - \Delta(\pi_w)  + 2R_{\textsf{max}}D_\textsf{TV}(\hat{\pi}_w\|\pi_w)\\
      &\quad+\left(1+(\alpha-1)D_\alpha(\pi_w\|\pi_\textsf{ref})\right)^{1/\alpha} \|\hat{r}-r^*\|_{L^\beta(\pi_\textsf{ref})}.
\end{aligned}
\end{equation}
Taking $\alpha=\beta=2$ leads to the desired result.

\clearpage
\subsection{Proof of Proposition~\ref{prop:bon-itp-tail}}
Fix a prompt $x$ and suppress the conditioning on it. 
Let the proxy reward be $\hat r := \hat r(x,Y) \in [0,1]$ with $Y \sim \pi_\textsf{ref}$. 
We define the endpoint gap $U := 1 - \hat r$. 
By assumption:
\begin{equation}\label{eq:U-beta-proof}
U \sim \text{Beta}\left(\frac{1}{\kappa}, 1\right), \quad \text{with density } f_U(u) = \frac{1}{\kappa} u^{\frac{1}{\kappa}-1} \text{ for } u \in [0,1].
\end{equation}
We utilize the general regret bound decomposition derived previously, where the regret $B(\pi)$ depends on the \textbf{Peak} (relative supremum), \textbf{Dist} (distortion), and \textbf{Gain} ($\Delta$) of the weighting function $w(\hat r)$:
\begin{equation}\label{eq:master-bound}
B(\pi) \leq 2\exp\left(-\frac{N}{\text{Peak}(w)}\right) + \epsilon \cdot \text{Dist}(w) - \Delta(w).
\end{equation}
The statistics are defined as:
\begin{equation}
\text{Peak}(w) = \frac{\sup_{\hat r} w(\hat r)}{\E[w(\hat r)]},\;\text{Dist}(w) = \frac{\sqrt{\E[w(\hat r)^2]}}{\E[w(\hat r)]},\; \Delta(w) = \Delta(\pi_w) = \E_{\pi_w}[\hat r] - \E[\hat r].
\end{equation}

\paragraph{Preliminaries: Moments.}
For $U \sim \text{Beta}(1/\kappa, 1)$, the raw moments are given by:
\begin{equation}\E[U^m] = \int_0^1 u^m \frac{1}{\kappa} u^{\frac{1}{\kappa}-1} du = \frac{1}{m\kappa + 1}.
\end{equation}
This implies that in the heavy-tail regime ($\kappa \to \infty$), $U$ remains distributed over $[0,1]$ but with mean $\approx 1/\kappa$. In the light-tail regime ($\kappa \to 0$), $\hat r = 1-U$ concentrates near 0 with moments:
\begin{equation}\E[\hat r] = 1 - \E[U] = \frac{\kappa}{\kappa+1}, \quad \E[\hat r^2] = 1 - \frac{2}{\kappa+1} + \frac{1}{2\kappa+1}.
\end{equation}
As $\kappa\to0$, for each fixed integer $m\geq 1$, $\E[\hat r^m]=m!\kappa^m + O(\kappa^{m+1})$.

\paragraph{Part 1: Heavy-tailed regime (large $\kappa$).}
Here, we analyze the statistics using $U$ as the primary variable.

\paragraph{1.1 Soft-BoN.} 
The weight is $w_\textsf{sBoN}(\hat r) = e^{\hat r/\lambda} = e^{(1-U)/\lambda} \propto e^{-U/\lambda}$. Let $\bar{w}(U) = e^{-U/\lambda}$.
\begin{itemize}
    \item \textbf{Peak:} 
        We expand the expectation of the weight using Taylor series:
        \begin{equation}\E[e^{-U/\lambda}] = 1 - \frac{1}{\lambda}\E[U] + O(\E[U^2]) = 1 - \frac{1}{\lambda(\kappa+1)} + O(\kappa^{-2}).
        \end{equation}
        The Peak is the ratio of the supremum (at $U=0$) to the expectation:
        \begin{equation}
        \text{Peak} = \frac{1}{\E[e^{-U/\lambda}]} = \frac{1}{1 - \frac{1}{\lambda\kappa} + o(\kappa^{-1})} = 1 + \frac{1}{\lambda\kappa} + o(\kappa^{-1}).
        \end{equation}
        Substituting this into the exponential term in Eq.~\eqref{eq:master-bound}:
        \begin{equation}
        \exp\left(-\frac{N}{\text{Peak}}\right) \approx \exp\left(-\frac{N}{1 + 1/(\lambda\kappa)}\right).
        \end{equation}
    \item \textbf{Dist.:}
        We compute the second moment $\E[w^2]$ using $\bar{w}^2 = e^{-2U/\lambda}$:
        \begin{equation}
        \E[e^{-2U/\lambda}] = 1 - \frac{2}{\lambda\kappa} + \frac{2}{\lambda^2(2\kappa)} + o(\kappa^{-1}) = 1 - \frac{2}{\lambda\kappa} + \frac{1}{\lambda^2\kappa} + o(\kappa^{-1}).
        \end{equation}
        The squared distortion is:
        \begin{equation}\text{Dist}^2 = \frac{\E[e^{-2U/\lambda}]}{(\E[e^{-U/\lambda}])^2} \approx \frac{1 - \frac{2}{\lambda\kappa} + \frac{1}{\lambda^2\kappa}}{1 - \frac{2}{\lambda\kappa} + \frac{1}{\lambda^2\kappa}(\frac{1}{2})} \approx 1 + \frac{1}{2\lambda^2\kappa}.
        \end{equation}
        Thus, $\text{Dist} \approx 1 + \frac{1}{4\lambda^2\kappa}$.
        \item \textbf{Gain:}
        The proxy gain is $\Delta = \E[U] - \E_{\pi_w}[U]$. 
        We compute the tilted expectation $\E_{\pi_w}[U]$ using the substitution $t = \kappa U$ (approximating $U$ as Exponential$(1/\kappa)$):
        \begin{equation}
        \E[U e^{-U/\lambda}] \approx \int_0^\infty u e^{-u/\lambda} \frac{1}{\kappa} e^{-u/\kappa} du \approx \frac{1}{\kappa} \int_0^1 u e^{-u/\lambda} du.
        \end{equation}
        A precise integration yields $\E_{\pi_w}[U] \approx \frac{\lambda}{\kappa}(1 - e^{-1/\lambda})$. 
        Thus:
        \begin{equation}
        \Delta \approx \frac{1}{\kappa} - \frac{\lambda}{\kappa}(1 - e^{-1/\lambda}) = \frac{1}{\kappa}\left( 1 - \lambda(1 - e^{-1/\lambda}) \right) + o(\kappa^{-1}).
        \end{equation}
\end{itemize}

\paragraph{1.2 ITP.}
The weight is $w_\textsf{ITP}(\hat r) = 1 + \hat r/\lambda = 1 + (1-U)/\lambda$. 
Let $C = 1 + 1/\lambda$. 
Then $w_\textsf{i} \propto C - U/\lambda$.
\begin{itemize}
    \item \textbf{Peak:}
    The supremum is $C$. The expectation is $\E[w_\textsf{i}] = C - \E[U]/\lambda \approx C - \frac{1}{\lambda\kappa}$.
    \begin{equation}
    \text{Peak} = \frac{C}{C - \frac{1}{\lambda\kappa}} = \frac{1}{1 - \frac{1}{C\lambda\kappa}} \approx 1 + \frac{1}{C\lambda\kappa} = 1 + \frac{1}{(\lambda+1)\kappa} + o(\kappa^{-1}).
    \end{equation}
    \item \textbf{Dist and Gain:}
    Using standard variance approximations for linear weights:
    \begin{equation}
    \text{Dist} \approx 1 + \frac{1}{4(\lambda+1)^2\kappa} + o(\kappa^{-1}), \quad \Delta \approx \frac{1}{2\kappa(\lambda+1)} + o(\kappa^{-1}).
    \end{equation}
\end{itemize}

\paragraph{Part 2: Light-tail regime (small $\kappa$).}
Here $\hat r$ concentrates near 0. 
Both soft-BoN ($e^{\hat r/\lambda}$) and ITP ($1+\hat r/\lambda$) behave as $1+\hat r/\lambda$ to the first order, and therefore they have similar terms of Peak and Dist.
The difference of the bound lies in the Gain.
\begin{itemize}
    \item \textbf{Peak:} For $w(\hat r) = 1 + \hat r/\lambda$:\begin{equation}\sup w = 1 + 1/\lambda = \frac{\lambda+1}{\lambda}, \quad \E[w] = 1 + \frac{\E[\hat r]}{\lambda} \approx 1 + \frac{\kappa}{\lambda} = \frac{\lambda+\kappa}{\lambda}.\end{equation}The Peak is:\begin{equation}
    \text{Peak} = \frac{(\lambda+1)/\lambda}{(\lambda+\kappa)/\lambda} = \frac{\lambda+1}{\lambda+\kappa}.
    \end{equation}
    Substituting into the bound exponent:
    \begin{equation}
    \exp\left(-\frac{N}{\text{Peak}}\right) = \exp\left(-\frac{N(\lambda+\kappa)}{\lambda+1}\right).
    \end{equation}
    \item \textbf{Dist.:} Using $\text{Var}(w) \approx \text{Var}(\hat r)/\lambda^2 \approx \kappa^2/\lambda^2$ and $\E[w] \approx 1$:\begin{equation}\text{Dist} \approx 1 + \frac{\text{Var}(w)}{2\E[w]^2} \approx 1 + \frac{\kappa^2}{2\lambda^2}.\end{equation}    
\end{itemize}

\paragraph{2.1 Soft-BoN.} 
The \textbf{Gain} is given by $\E_{\pi_w}[\hat r] = \E[\hat r e^{\hat r/\lambda}] / \E[e^{\hat r/\lambda}]$. 
Expanding the numerator $\E[\hat r e^{\hat r/\lambda}]$ yields 
\begin{equation}
    \E[\hat r(1 + \hat r/\lambda + \frac{1}{2\lambda^2}\hat r^2 + O(\hat r^3))] = \E[\hat r] + \frac{1}{\lambda}\E[\hat r^2] + \frac{1}{2\lambda^2}\E[\hat r^3] + O(\E[\hat r^4]).
\end{equation}
By the moments of $\hat{r}$ with small $\kappa$, we obtain
\begin{equation}
    \E[\hat r e^{\hat r/\lambda}] = \kappa + \frac{2\kappa^2}{\lambda} + \frac{3\kappa^3}{\lambda^2} + O(\kappa^4),
\end{equation}
while the denominator expands as $\E[e^{\hat r/\lambda}] = 1 + \frac{\kappa}{\lambda} + \frac{\kappa^2}{\lambda^2} + O(\kappa^3)$. Performing a ratio expansion then yields $\E_{\pi_w}[\hat r] = \kappa + \frac{\kappa^2}{\lambda} + \frac{\kappa^3}{\lambda^2} + O(\kappa^4)$. Finally, subtracting the reference mean $\E[\hat r] = \kappa + O(\kappa^2)$ gives the proxy gain 
\begin{equation}
    \Delta = \frac{\kappa^2}{\lambda} + \frac{\kappa^3}{\lambda^2} + O(\kappa^4).
\end{equation}

\paragraph{2.2 ITP.} 
The \textbf{Gain} is dominated by the covariance term:\begin{equation}
\Delta = \frac{\text{Cov}(w, \hat r)}{\E[w]} \approx \frac{\text{Var}(\hat r)/\lambda}{1} \approx \frac{\kappa^2}{\lambda} +O(\kappa^3).
\end{equation}

\paragraph{Part 3: BoN limit ($\lambda\to0$).}
Best-of-$N$ (BoN) arises as the limit of Soft-BoN as the inverse temperature $1/\lambda \to \infty$. Since the supremum of the weights is $\sup_y w_\mathrm{soft}(\hat r(x,y)) = e^{1/\lambda}$ and the normalizing constant $\E[e^{\hat r/\lambda}]$ remains finite for any fixed $\kappa$, the peak ratio diverges as $\mathrm{Peak}(w_\mathrm{soft};x) = e^{1/\lambda} / \E[e^{\hat r/\lambda}] \to \infty$. Similarly, the distortion diverges as $\mathrm{Dist}(w_\mathrm{soft};x) = \sqrt{\E[e^{2\hat r/\lambda}]} / \E[e^{\hat r/\lambda}] \to \infty$. In the light-tail regime, the expansion $\E[e^{\hat r/\lambda}] = 1 + \kappa/\lambda + \kappa^2/\lambda^2 + O(\kappa^3)$ implies that the peak scales as $\mathrm{Peak}(w_\mathrm{soft};x) \asymp e^{1/\lambda}$. Consequently, minimizing the finite-$N$ proxy term $2\exp(-N/\mathrm{Peak}) \approx 2\exp(-N/e^{1/\lambda})$ requires the sample size $N$ to scale on the order of $e^{1/\lambda}$, confirming the qualitative behavior of BoN.

Combining Parts 1-3 completes the proof.

\clearpage
\subsection{Proof of Proposition~\ref{prop:alpha-rlhf}}
Consider the generator function $\phi_\alpha(t)$ of the $\alpha$-Tsallis divergence 
\begin{equation}
\begin{aligned}
    \phi_\alpha(t) = \frac{t^\alpha-\alpha t+(\alpha-1)}{\alpha(\alpha-1)},\;t>0,\alpha>1.
\end{aligned}
\end{equation}
Then, the $\alpha$-Tsallis divergence between two distributions $P$ and $Q$ can be alternatively expressed as $D_\alpha(P\|Q) = \sum\limits_{x\in\calX} Q(x) \phi_\alpha\left( \frac{P(x)}{Q(x)} \right)$.
Note that $\phi_\alpha(t)$ is convex with $\phi_\alpha(1)=0$, and therefore $D_\alpha$ is also convex in $P$.
Now, let the tilting function be $t(y|x) = \pi(y|x)/\pi_\textsf{ref}(y|x)$, the objective in Eq.~\eqref{eq:alpha-rlhf} becomes 
\begin{equation}\label{app:eq:convex-conjugate}
\begin{aligned}
     \calL_\alpha(\pi) 
     & = \E_{y\sim\pi(\cdot|x)}\left[ \hat{r}(x, y) \right] - \lambda D_\alpha\left(\pi(\cdot|x)\|\pi_\textsf{ref}(\cdot|x)\right)\\
     & = \sum\limits_{y\in\calY} \pi(y|x)\hat{r}(x, y) - \lambda \sum\limits_{y\in\calY} \pi_\textsf{ref}(y|x) \phi_\alpha\left(\frac{\pi(y|x)}{\pi_\textsf{ref}(y|x)}\right)\\
     & = \sum\limits_{y\in\calY} \pi_\textsf{ref}(y|x) \left( t(y|x)\hat{r}(x, y) - \lambda \phi_\alpha(t(y|x)) \right).
\end{aligned}
\end{equation}
The simplex constraint $\sum\limits_{y\in\calY} \pi(y|x) = 1$ becomes $\sum\limits_{y\in\calY} \pi_\textsf{ref}(y|x) t(y|x) = 1$.
The transformation in Eq.~\eqref{app:eq:convex-conjugate} actually forms the convex conjugate (Fenchel–Legendre transform) of the generator function, i.e.,
\begin{equation}
\begin{aligned}
    \psi_\alpha(z) = \sup\limits_{t>0}\left\{ zt-\lambda \phi_\alpha(t) \right\}.
\end{aligned}
\end{equation}
The optimizer $t^*(z) = \argmax\limits_{t > 0} zt-\lambda \phi_\alpha(t)$ can then be derived directly from the first-order condition
\begin{equation}
\begin{aligned}
    \frac{d}{dt} (zt - \phi_\alpha(t)\lambda ) &= z - \lambda \phi_\alpha'(t) = z - \lambda \frac{1}{\alpha(\alpha-1)} (\alpha t^{\alpha-1}-\alpha) = z - \lambda \frac{t^{\alpha-1}-1}{\alpha-1} = 0.
\end{aligned}
\end{equation}
By re-arranging and plugging $z = \hat{r}(x, y)$, we have
\begin{equation}
    t^{\alpha-1} = 1+\frac{\alpha-1}{\lambda} \hat{r}(x, y) \Rightarrow t^*(\hat{r}(x, y)) = \left[  1+(\alpha-1)\frac{\hat{r}(x, y)}{\lambda}\right]_+^{\frac{1}{\alpha-1}} = \exp_\alpha\left( \frac{\hat{r}(x, y)}{\lambda} \right),
\end{equation}
and therefore 
\begin{equation}
    \pi_\alpha(\cdot|x)\propto \pi_\textsf{ref}(\cdot|x) \exp_\alpha\left(\frac{\hat{r}(x, y)}{\lambda}\right).
\end{equation}
The simplex condition $\sum\limits_{y\in\calY} \pi_\alpha(y|x) = 1$ can be satisfied by finding a $\nu \geq 0$ such that $\sum\limits_{y\in\calY} \pi_\textsf{ref}(y|x) t^*(\hat{r}(x, y)+\nu)=1$.

\clearpage
\subsection{Proof of Proposition~\ref{prop:trend-a-kappa-lambda}}
We proceed in three parts.
\paragraph{Part 1: Algebraic reductions.}
Let $U:=1-\hat{r}$ and note $\hat{r}=1-U$. We use the parameters from the proposition statement: $\alpha \in (1,2]$ and $\lambda > 0$. The weight function is:
\begin{equation}
w_\alpha(\hat{r}) = \left[1+(\alpha-1)\frac{\hat{r}}{\lambda}\right]_+^{\frac{1}{\alpha-1}} = \left(1+\frac{\alpha-1}{\lambda}(1-U)\right)^{\frac{1}{\alpha-1}}.
\end{equation}

We factor out the endpoint value at $\hat{r}=1$ (i.e., $U=0$). 
Define the constants:\begin{equation}
A := 1+\frac{\alpha-1}{\lambda},\;\rho := \frac{(\alpha-1)/\lambda}{1+(\alpha-1)/\lambda} = \frac{\alpha-1}{\lambda A},\;\gamma := \frac{1}{\alpha-1}.
\end{equation}

Then we can rewrite the inner term as:
\begin{equation}
1+\frac{\alpha-1}{\lambda}(1-U) = 1 + \frac{\alpha-1}{\lambda} - \frac{\alpha-1}{\lambda}U = A - A\rho U = A(1-\rho U).
\end{equation}
Thus, the weight becomes:
\begin{equation}
w_\alpha(1-U) = \left(A(1-\rho U)\right)^\gamma = A^\gamma(1-\rho U)^\gamma.
\end{equation}

The moments of the weights under the reference distribution (parameterized by $\kappa$) are:
\begin{equation}\label{eq:m0-def}
\E_\kappa[w_\alpha(\hat{r})] = A^\gamma \E_\kappa[(1-\rho U)^\gamma] = A^\gamma m_0,\;\E_\kappa[w_\alpha(\hat{r})^2] = A^{2\gamma} \E_\kappa[(1-\rho U)^{2\gamma}] = A^{2\gamma} m_2,
\end{equation}
where we define $m_0(\alpha;\kappa,\lambda) := \E_\kappa[(1-\rho U)^\gamma]$ and $m_2(\alpha;\kappa,\lambda) := \E_\kappa[(1-\rho U)^{2\gamma}]$.
Therefore, the distortion ratio simplifies to:
\begin{equation}\label{eq:dist-ratio-m}\frac{|w_\alpha|{L^2}}{|w_\alpha|{L^1}} = \frac{\sqrt{\E\kappa[w_\alpha(\hat{r})^2]}}{\E_\kappa[w_\alpha(\hat{r})]} = \frac{\sqrt{A^{2\gamma}m_2}}{A^\gamma m_0} = \frac{\sqrt{m_2}}{m_0}.
\end{equation}

Since $w_\alpha(\hat{r})$ is increasing in $\hat{r}$, its supremum is at $\hat{r}=1$, so $\sup w_\alpha = A^\gamma$. 
The envelope constant for the exponential mechanism is:
\begin{equation}
M(\alpha) = \frac{\sup w_\alpha}{\E_\kappa[w_\alpha]} = \frac{A^\gamma}{A^\gamma m_0} = \frac{1}{m_0}.
\end{equation}

The probability of selecting the reference policy is $e^{-N M(\alpha)^{-1}} = e^{-N m_0}$. 
This yields the first term in \eqref{eq:B-def-clean}.
The proxy gain is:
\begin{equation}
\Delta(\pi_\textsf{BoT}) = \E_{\pi_\alpha}[\hat{r}] - \E_{\pi_\textsf{ref}}[\hat{r}] = \E_{\pi_\textsf{ref}}[U] - \E_{\pi_\alpha}[U],
\end{equation}
where we used $\hat{r}=1-U$.

\paragraph{Part 2: Heavy-tailed regime (large $\kappa$).}
In the heavy-tailed regime ($\kappa \to \infty$), $U$ concentrates near $0$. 
Using $U \sim \text{Beta}(1/\kappa, 1)$, we have the moments $\E[U] = \frac{1}{\kappa+1} = \frac{1}{\kappa} - \frac{1}{\kappa^2} + O(\kappa^{-3})$ and $\E[U^2] = \frac{1}{2\kappa+1} = \frac{1}{2\kappa} + O(\kappa^{-2})$.
We expand $(1-\rho U)^\gamma$ around $U=0$:
\begin{equation}
(1-\rho U)^\gamma = 1 - \gamma\rho U + \frac{\gamma(\gamma-1)}{2}\rho^2 U^2 + O(U^3).
\end{equation}
Taking expectations:
\begin{equation}
m_0 = 1 - \gamma\rho\E[U] + \frac{\gamma(\gamma-1)}{2}\rho^2\E[U^2] + O(\kappa^{-2}) = 1 + \frac{1}{\kappa}\left(-\gamma\rho + \frac{\gamma(\gamma-1)}{4}\rho^2\right) + O(\kappa^{-2}).
\end{equation}
Substituting $\gamma = \frac{1}{\alpha-1}$ and $\rho = \frac{(\alpha-1)/\lambda}{A}$:
\begin{equation}
\gamma\rho = \frac{1}{\alpha-1} \cdot \frac{\alpha-1}{\lambda A} = \frac{1}{\lambda A}.
\end{equation}
The quadratic coefficient involves $\gamma(\gamma-1) = \frac{1}{\alpha-1}(\frac{1}{\alpha-1}-1) = \frac{2-\alpha}{(\alpha-1)^2}$. 
Thus:
\begin{equation}
\frac{\gamma(\gamma-1)}{4}\rho^2 = \frac{2-\alpha}{4(\alpha-1)^2} \frac{(\alpha-1)^2}{\lambda^2 A^2} = \frac{2-\alpha}{4\lambda^2 A^2}.
\end{equation}
Combining these over a common denominator $4\lambda^2 A^2 = 4(\lambda+\alpha-1)^2$:
\begin{equation}m_0 = 1 + \frac{1}{\kappa} \frac{-4\lambda A + 2-\alpha}{4\lambda^2 A^2} + O(\kappa^{-2}) = 1 + \frac{1}{\kappa} \frac{-4\lambda - 5(\alpha-1) + 1}{4(\lambda+\alpha-1)^2} + O(\kappa^{-2}).
\end{equation}
The finite-$N$ term expands as $e^{-N m_0} \approx e^{-N}(1 - N(m_0-1))$, contributing the term matching \eqref{eq:B-large-kappa-clean}.

Expanding $m_2 = \E[(1-\rho U)^{2\gamma}]$ similarly leads to $m_2 = 1 + \frac{1}{\kappa}(-2\gamma\rho + \frac{2\gamma(2\gamma-1)}{4}\rho^2) + O(\kappa^{-2})$.
The ratio $\sqrt{m_2}/m_0$ expands as:
\begin{equation}
\frac{\sqrt{m_2}}{m_0} \approx 1 + \frac{1}{\kappa}\left(\frac{1}{2}\text{coeff}(m_2) - \text{coeff}(m_0)\right) = 1 + \frac{1}{\kappa}\frac{\gamma^2 \rho^2}{4}.
\end{equation}
Substituting $\gamma\rho = \frac{1}{\lambda+\alpha-1}$:
\begin{equation}
\frac{\sqrt{m_2}}{m_0} = 1 + \frac{1}{4\kappa(\lambda+\alpha-1)^2}.
\end{equation}
This yields the first term in \eqref{eq:B-large-kappa-clean}.

We compute $\Delta(\pi_\textsf{BoT}) = \E[U] - \E_{\pi_\alpha}[U]$.
Using the integral approximation for large $\kappa$:
\begin{equation}
\E_{\pi_\alpha}[U] = \frac{1}{\kappa} \int_0^1 (1-\rho u)^\gamma du + O(\kappa^{-2}) = \frac{1}{\kappa} \frac{1-(1-\rho)^{\gamma+1}}{\rho(\gamma+1)} + O(\kappa^{-2}).
\end{equation}
Linearizing near $\alpha=2$ (where $\gamma=1$), the integral is $1-\rho/2$. 
Thus:\begin{equation}
\Delta(\pi_\textsf{BoT}) \approx \frac{1}{\kappa} - \frac{1}{\kappa}(1-\rho/2) = \frac{\rho}{2\kappa} = \frac{1}{\kappa} \frac{\alpha-1}{2(\lambda+\alpha-1)}.
\end{equation}
This matches the final term in \eqref{eq:B-large-kappa-clean}.

\paragraph{Part 3: Light-tailed regime (small $\kappa$).}
In this regime ($\kappa \to 0$), $\hat{r}$ concentrates near $0$.
From $U \sim \text{Beta}(1/\kappa, 1)$, $\hat{r} = 1-U$ has moments:
\begin{equation}
\E[\hat{r}] \approx \kappa - \kappa^2 + \kappa^3, \quad \E[\hat{r}^2] \approx 2\kappa^2 - 6\kappa^3.
\end{equation}
Using the expansion $w_\alpha(\hat{r}) \approx 1 + \frac{\hat{r}}{\lambda} + \frac{2-\alpha}{2\lambda^2}\hat{r}^2$:
\begin{equation}\E[w_\alpha] \approx 1 + \frac{\kappa}{\lambda} + \kappa^2\left(\frac{2-\alpha}{\lambda^2} - \frac{1}{\lambda}\right),\;\E[w_\alpha^2] \approx 1 + \frac{2\kappa}{\lambda} + \kappa^2\left(\frac{2(3-\alpha)}{\lambda^2} - \frac{2}{\lambda}\right).
\end{equation}
The gain $\Delta = \E_{\pi_\alpha}[\hat{r}] - \E[\hat{r}]$ is dominated by the covariance term $\text{Cov}(w_\alpha, \hat{r})/\E[w_\alpha]$, and $\E_{\pi_\alpha}[\hat{r}] \approx \kappa + \frac{\kappa^2}{\lambda}$.
Thus $\Delta \approx \frac{\kappa^2}{\lambda}$. 
Including higher order terms leads to the expansion in \eqref{eq:B-small-kappa-clean}.
The squared distortion is:\begin{equation}
\text{Dist}^2 = \frac{\E[w_\alpha^2]}{\E[w_\alpha]^2} \approx 1 + \frac{\kappa^2}{\lambda^2}.
\end{equation}
Taking the square root yields $\text{Dist} \approx 1 + \frac{\kappa^2}{2\lambda^2}$.
Combining these expansions yields the asymptotic results stated in the proposition.

\clearpage
\section{Additional Discussion and Experimental Setup}\label{app:sec-additional}
We detail the Hill estimator in \S~\ref{sec:bot-implement}, elucidating its sample complexity advantages over estimating the entire reward distribution and outlining our experimental implementation.

\subsection{Tail Estimation via the Hill Estimator}\label{app:sec-hill}
To implement the prompt-adaptive strategy, we require a robust estimate of the tail index $\kappa(x)$ from a finite set of $N$ candidate responses. 
We adopt the Hill estimator \citep{hill1975simple}, the standard semi-parametric estimator for heavy-tailed distributions.
Standard extreme value theory typically deals with unbounded random variables; however, our proxy rewards are bounded, $\hat{r}(x,y) \in [0, 1]$. 
To bridge this, we analyze the reward gaps (or regrets):\begin{equation}
U(x,y) = 1 - \hat{r}(x,y).
\end{equation}
The ``tail'' of the reward distribution (where $\hat{r} \approx 1$) corresponds to the limit $U \to 0$. 
As established in our modeling assumption (Prop.~\ref{prop:bon-itp-tail}), the tail behavior is characterized by $\Pr(U \le u) \approx u^{1/\kappa}$.
By applying the transformation $Z = 1/U = (1 - \hat{r})^{-1}$, we map the problem to a standard Pareto tail estimation where $\Pr(Z > z) = \Pr(U < 1/z) \approx z^{-1/\kappa}$. 
Here, $\kappa$ appears precisely as the shape parameter (tail index) of the resulting Pareto distribution.

Given $N$ sampled responses, let $r_{(1)} \ge r_{(2)} \ge \dots \ge r_{(N)}$ denote the order statistics of the rewards. 
These correspond to the ordered transformed variables $Z_{(1)} \ge Z_{(2)} \dots$, where $Z_{(i)} = (1 - r_{(i)})^{-1}$.
The Hill estimator uses the top $K$ order statistics to estimate $\kappa$:
\begin{equation}\label{eq:hill-derivation}
\begin{aligned}
\hat{\kappa} &= \frac{1}{K} \sum_{i=1}^K \log \left( \frac{Z_{(i)}}{Z_{(K+1)}} \right) = \frac{1}{K} \sum_{i=1}^K \log \left( \frac{(1 - r_{(i)})^{-1}}{(1 - r_{(K+1)})^{-1}} \right) = \frac{1}{K} \sum_{i=1}^K \log \left( \frac{1 - r_{(K+1)}}{1 - r_{(i)}} \right).
\end{aligned}
\end{equation}

The hyperparameter $K$ determines the ``tail cut-off''---the boundary between the distribution's tail and its body.
\begin{itemize}
\item Small $K$: Uses only the most extreme data points. This minimizes bias (as the asymptotic tail approximation is most accurate here) but increases variance due to the small sample size.
\item Large $K$: Incorporates more samples, reducing variance. However, if $K$ is too large, it includes samples from the central body of the distribution where the power-law assumption fails, introducing significant bias.
\end{itemize}
In \texttt{BoT}, we fix $K$ to a small fraction of $N$ (e.g., $K \approx \sqrt{N}$ or $5\%-10\%$) to prioritize low bias, ensuring $\alpha(x)$ reacts only to the true extremal behavior of the prompt.

\clearpage
\subsection{Sample Complexity: Tail Index vs. Full Distribution}\label{app:sec-tail-vs-full}
A fundamental advantage of the \texttt{BoT} framework is that it adapts to the reward landscape without requiring a full estimate of the reward distribution. 
Estimating the complete probability density function $p(r)$ is a non-parametric problem subject to the curse of dimensionality and slow convergence rates. 
In contrast, estimating the tail index $\kappa$ is a semi-parametric problem that enjoys faster statistical convergence. 
We formalize this distinction below.

Consider the task of estimating the underlying reward density $p(r)$ from $N$ samples to within an error $\epsilon$ in total variation distance or $L_2$ norm. 
The minimax convergence rate for estimating a generic smooth density (Hölder class $\Sigma(\beta, L)$) in $d=1$ dimension is
\begin{equation}
\inf_{\hat{p}} \sup_{p \in \Sigma} \E \left[ | \hat{p} - p | \right] \asymp N^{-\frac{\beta}{2\beta+1}},
\end{equation}
where $\beta$ is the smoothness parameter. 
For typical values (e.g., $\beta=2$), this yields a rate of $N^{-2/5}$, which is significantly slower than the parametric rate $N^{-1/2}$. 
This implies that accurately characterizing the entire reward landscape to optimize a policy would require a prohibitively large number of samples $N$.

Conversely, \texttt{BoT} relies only on the tail index $\kappa$, a single scalar summary of the extreme value behavior. 
Under standard Second-Order Regular Variation conditions, the Hill estimator $\hat{\kappa}$ (using the top $K$ statistics) is asymptotically normal \citep{de1998asymptotic}:
\begin{equation}
\sqrt{K}\left( \hat{\kappa}_K - \kappa \right) \xrightarrow{d} \mathcal{N}(0, \kappa^2).
\end{equation}
This indicates that the estimation error decays at the parametric rate $O(1/\sqrt{K})$. Crucially, the sample complexity to estimate $\kappa$ to a fixed precision $\epsilon$ depends on the number of \emph{tail samples} $K$, not the complexity of the full distribution's support.

In our design, $\alpha(x)$ (Eq.~\eqref{eq:hill-to-alpha}) is a smooth, bounded monotonic function of $\kappa$, we do not need infinite precision in $\hat{\kappa}$. 
We essentially perform a ``soft classification'' of the tail regime (heavy vs. light). 
The sample complexity to distinguish these regimes is low; even with moderate sample sizes (e.g., $N \approx 50$), the variance of $\hat{\kappa}$ is sufficiently small to steer $\alpha(x)$ towards the correct end of the interval $[1, 2]$. 
This makes \texttt{BoT} highly sample-efficient compared to methods that might attempt to model the full reward surface.

\clearpage
\subsection{Experimental Setup}\label{app:sec-exp-expsetup}
We conduct a comprehensive empirical analysis across a diverse suite of reasoning and instruction-following tasks. Our evaluation protocol encompasses four standard benchmarks, three instruction-tuned reference policies, and four diverse reward models.

We evaluate performance on the following four benchmarks:
\begin{enumerate}
    \item GSM8K (Grade School Math) \citep{cobbe2021training}: A standard benchmark for multi-step mathematical reasoning. We utilize the test split, comprising 1,319 high-quality grade-school math problems with integer answers. We define the oracle reward $r^\star(x, y)$ as a binary metric: $r^\star(x, y) = 1$ if the final numerical answer is correct, and 0 otherwise.
    
    \item MMLU (Massive Multitask Language Understanding) \citep{hendrycks2020measuring}: To assess knowledge-intensive reasoning in specialized STEM domains, we utilize the \emph{College Mathematics} and \emph{College Chemistry} test splits (approx. 100 questions each). Unlike GSM8K, these tasks use a multiple-choice format. Correctness is determined by matching the model's final selection to the ground truth key.
    
    \item MATH (Mathematics Aptitude Test of Heuristics) \citep{hendrycks2021measuring}: For competition-level mathematics requiring complex heuristics, we construct an evaluation set by randomly sampling 512 instances from the MATH test split. These problems are more challenging, and require complex heuristic application and advanced reasoning beyond standard grade-school arithmetic. Similar to GSM8K, the oracle reward is binary based on the exact match of the final calculated result.
    
    \item AlpacaFarm \citep{dubois2023alpacafarm}: Designed to simulate "in-the-wild" interactions, this dataset covers diverse tasks from creative writing to open-ended QA. Following \citet{coste2023reward}, we use the standard test set of 805 prompts. Since open-ended tasks lack verifiable ground truth, we employ \texttt{Alpaca-RM} \citep{dubois2023alpacafarm}---fine-tuned from LLaMA-7B on human preference data—as the oracle reward $r^\star$ for evaluation.
\end{enumerate}

For the reasoning tasks (GSM8K, MMLU, MATH) with verifiable ground truth, we consider three instruction-tuned reference policies ($\pi_\textsf{ref}$): \texttt{Gemma-2-2B} \citep{team2024gemma}, \texttt{Llama-3-3B} \citep{grattafiori2024llama}, and \texttt{Mistral-7B} \citep{jiang2023mistral}. 
For AlpacaFarm, we follow prior work and use the \texttt{Pythia-1.4B} model \citep{biderman2023pythia}.

We pair these policies with four diverse proxy reward models ($\hat{r}$):
\begin{enumerate}
\item \texttt{OASST-RM} (1.4B): Based on Pythia-1.4B \citep{kopf2023openassistant}.
\item \texttt{Gemma-RM} (2B): Trained from \texttt{google/gemma-2b-it} \citep{dong2023raft}.
\item \texttt{Llama-RM} (3B): A generalizable RM trained from \texttt{meta/llama-3.2-3B} \citep{yang2024regularizing}.
\item \texttt{ARMO-RM} (8B): A mixture-of-experts model trained from \texttt{meta/llama-3-8B} \citep{wang2024interpretable}.
\end{enumerate}

We prompt all reference policies using zero-shot Chain-of-Thought (CoT) \citep{wei2022chain}, i.e., without any in-context demonstrations, to elicit reasoning. 
To maximize response diversity, we set the temperature to 1.0 and top-$p$ to 1.0 (no nucleus sampling), limiting generation to 512 tokens.
For each unique combination of dataset, policy, and reward model, we pre-generate a large candidate pool of $2^{12} = 4,096$ responses per prompt.
For AlpacaFarm, we generate 12,600 responses per prompt following \citealp{coste2023reward}. 
We simulate experiments by bootstrapping subsets of varying size $N$ from this pool, reporting results averaged over 10 independent trials to ensure statistical robustness.

For the Hill estimator (Eq.~\eqref{eq:hill-estimator}) in \texttt{BoT}, we adaptively set the tail cutoff as $K=\max(1, \sqrt{N})$ to balance bias and variance as $N$ scales. 
To calibrate the mapping function $\alpha(x)$, we select the pivot $\kappa_0$ based on the reward model's typical behavior; specifically, we set $\kappa_0$ to the median of the estimated tail indices $\{\hat{\kappa}(x_i)\}_{i=1}^M$ over the set of $M$ prompts. This centers the sigmoid mapping, preventing the policy from saturating into purely optimistic or pessimistic modes, and ensuring \texttt{BoT} remains sensitive to relative tail risk variations across prompts.

\clearpage 
\section{Additional Experiments and Ablation Studies}\label{app:sec:additional-and-ablation}
We include additional experimental results and comparisons, including: (i) comprehensive results of all configurations of the reference policies and reward models for each datasets, similar to Figures~\ref{fig:main-exp} and~\ref{fig:alpacafarm-exp} in the main text; (ii) additional results of selected configurations in Figures~\ref{fig:main-exp}, including true reward $r^*$ over $N$, proxy reward $\hat{r}$ over $N$. and the proxy reward over distortions (i.e., the RLHF objective in Eq.~\eqref{eq:alpha-rlhf}); (iii) ablation studies on the steering temperature $\lambda$ and tail pivot $\kappa_0$ in Algorithm~\ref{alg:bot}.

\subsection{The Complete Results over all Datasets, Reference Policies, and Reward Models}\label{app:sec:all-exp}
We show the trajectory of the true reward ($r^*$) versus the proxy reward ($\hat{r}$) as the sample size $N$ scales exponentially from $2^0$ to $2^{10}$, i.e., $N\in\{2^0. 2^1, \cdots, 2^9, 2^{10}\}$. 
These detailed results span all combinations of reference policies and reward models, i.e., \{ \texttt{Gemma-2-2B}, \texttt{Llama-3-3B}, \texttt{Mistral-7B} \} $\times$ \{ \texttt{OASST-RM}, \texttt{Gemma-RM}, \texttt{Llama-RM}, \texttt{ARMO-RM} \} for GSM8K (Figure~\ref{fig:app:gsm8k-all-acc-vs-reward}), MMLU (Figure~\ref{fig:app:mmlu-all-acc-vs-reward}), MATH (Figure~\ref{fig:app:math-all-acc-vs-reward}), and \{ \texttt{Pythia-1.4B}\} $\times$ \{ \texttt{OASST-RM}, \texttt{Gemma-RM}, \texttt{Llama-RM}, \texttt{ARMO-RM} \} for AlpacaFarm (Figure~\ref{fig:app:alpaca-all-acc-vs-reward}).

Across these configurations, we observe distinct failure modes for the baselines: optimistic strategies, such as \texttt{BoN} and \texttt{sBoN}, consistently succumb to reward hacking at large $N$, while \texttt{ITP} remains overly conservative, often failing to adequately explore high-reward regions—a finding consistent with \citet{huang2025best}. In contrast, \texttt{BoT} demonstrates robust adaptivity. Depending on the reward landscape, it either outperforms all baselines in high-variance regimes (e.g., \texttt{Mistral-7B}/\texttt{Gemma-RM} on MATH), matches the efficiency of \texttt{BoN} in light-tailed regime (e.g., \texttt{Llama-3-3B}/\texttt{Llama-RM} on MMLU), or converges to the safety of \texttt{ITP} when heavy-tailed rewards are dominant (e.g., \texttt{Gemma-2-2B}/\texttt{Gemma-RM} on GSM8K).

\begin{figure*}[b!]
  \vskip 0.2in
  \begin{center}
    \centerline{\includegraphics[width=.9\textwidth]{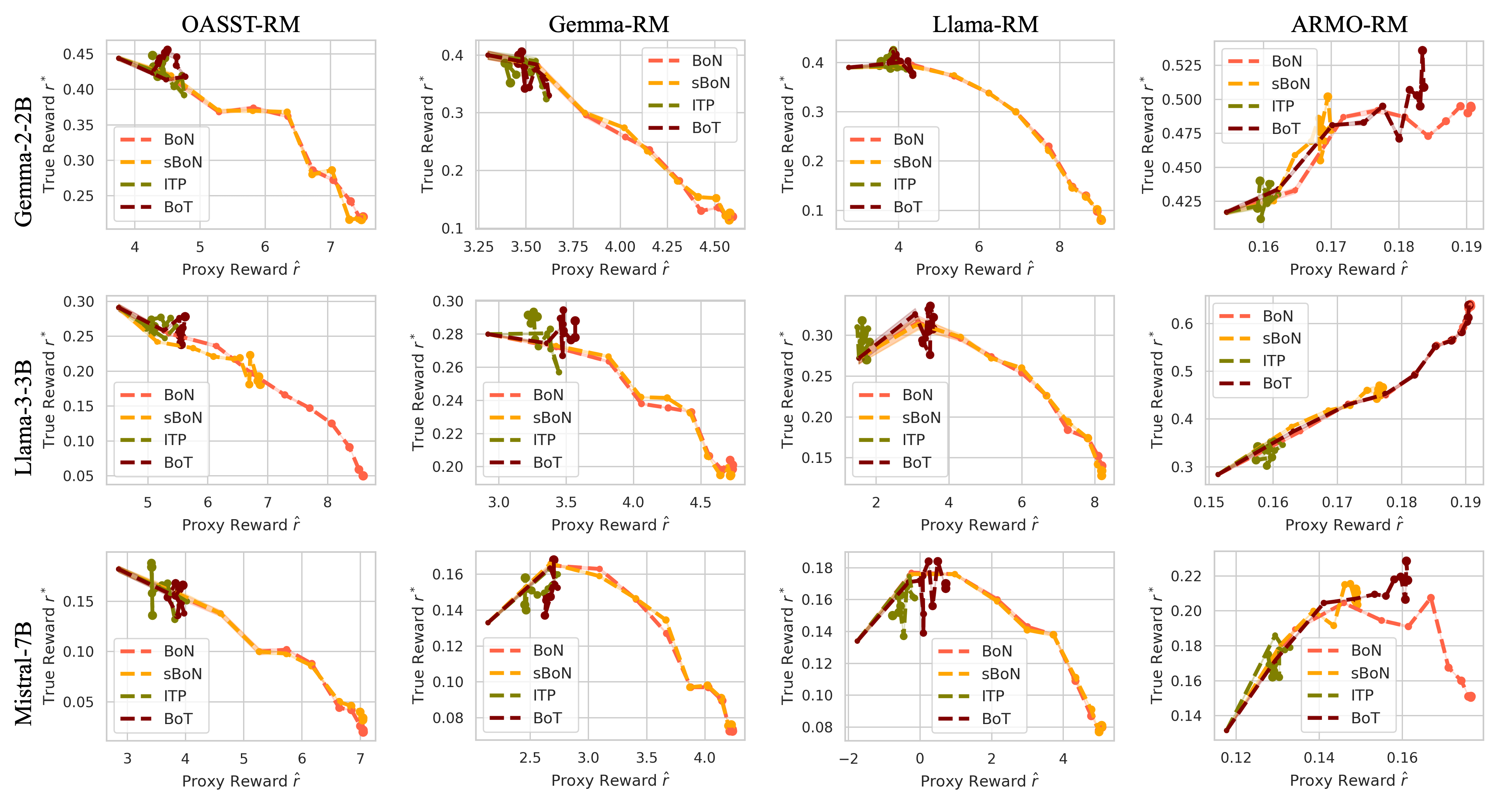}}
    \caption{\small 
    Comparison of optimistic (\texttt{sBoN}/\texttt{BoN}), pessimistic (\texttt{ITP}), and adaptive (\texttt{BoT}) strategies on GSM8K across different reference models and proxy rewards. 
    Each curve traces the trajectory of True Reward ($r^*$) vs. Proxy Reward ($\hat{r}$) as the sample size $N$ increases from $2^0$ to $2^{10}$, with marker size proportional to $N$. The smallest marker represents $N=1$ (standard sampling), while the largest represents $N=1024$.
    }
    \label{fig:app:gsm8k-all-acc-vs-reward}
  \end{center}
\end{figure*}

\clearpage
\begin{figure*}[b!]
  \begin{center}
    \centerline{\includegraphics[width=.9\textwidth]{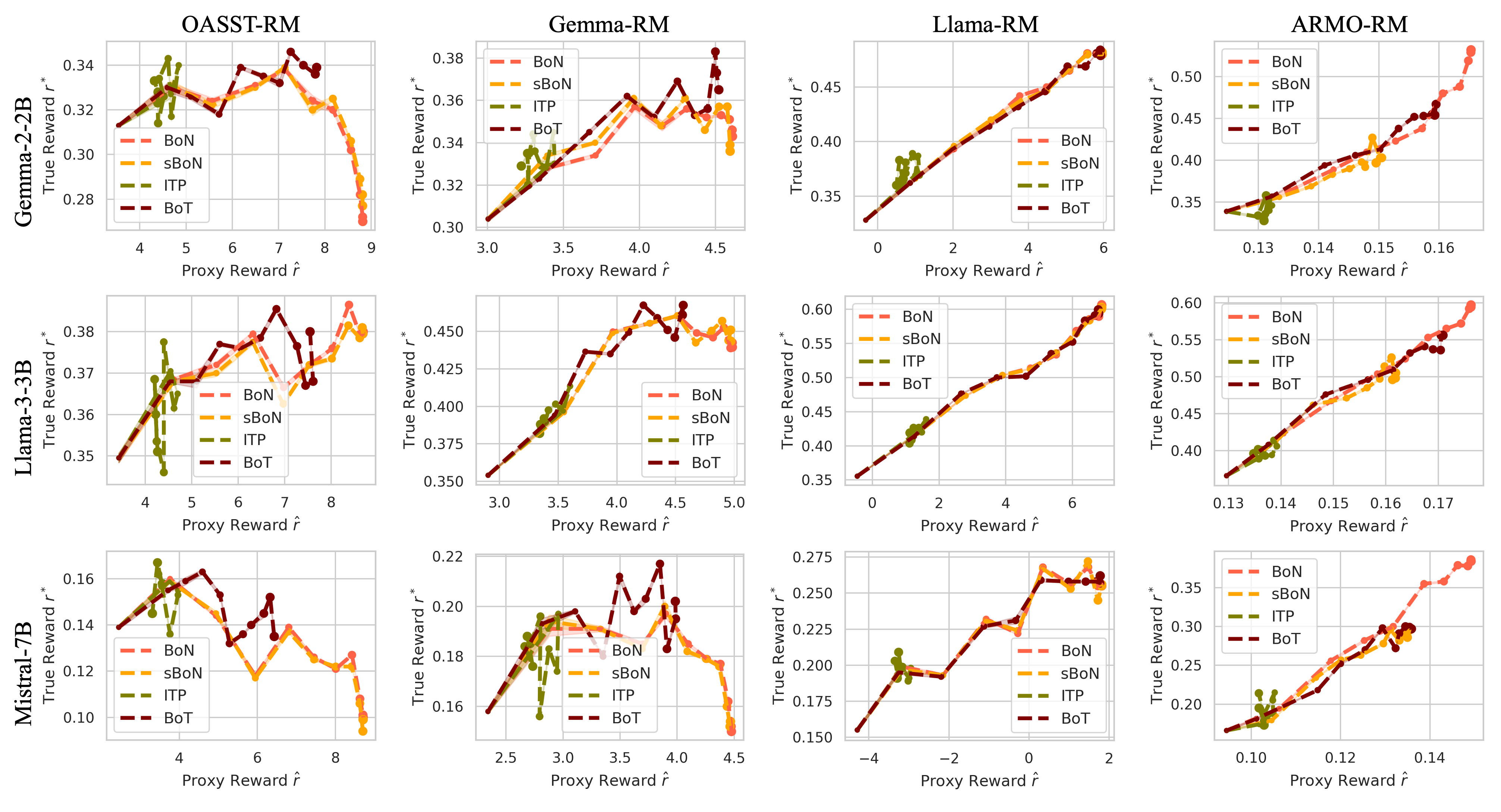}}
    \caption{\small 
    Comparison of optimistic (\texttt{sBoN}/\texttt{BoN}), pessimistic (\texttt{ITP}), and adaptive (\texttt{BoT}) strategies on MMLU across different reference models and proxy rewards. 
    Each curve traces the trajectory of True Reward ($r^*$) vs. Proxy Reward ($\hat{r}$) as the sample size $N$ increases from $2^0$ to $2^{10}$, with marker size proportional to $N$. The smallest marker represents $N=1$ (standard sampling), while the largest represents $N=1024$.
    }
    \label{fig:app:mmlu-all-acc-vs-reward}
  \end{center}
\end{figure*}

\begin{figure*}[b!]
  \begin{center}
    \centerline{\includegraphics[width=.9\textwidth]{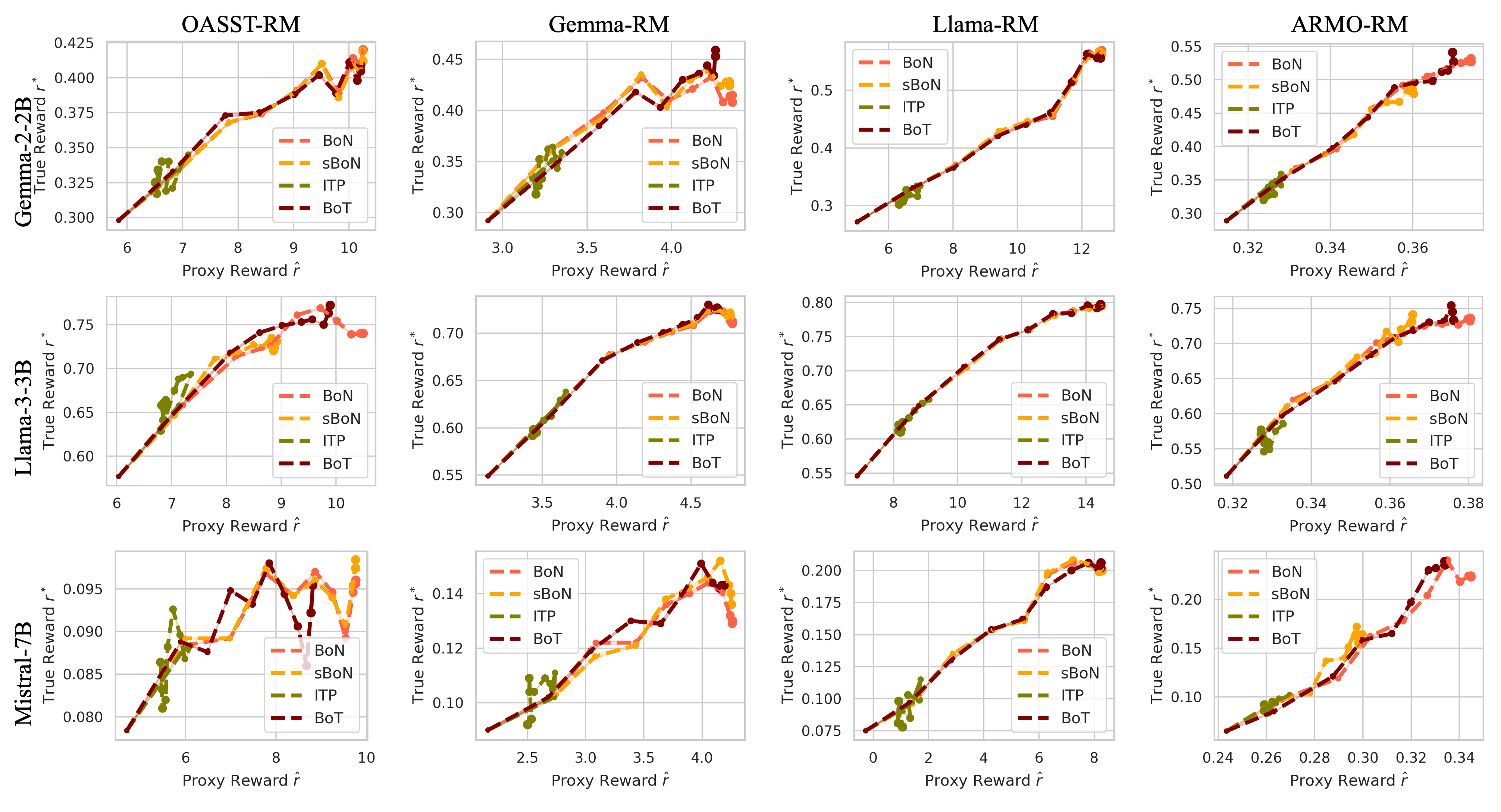}}
    \caption{\small 
    Comparison of optimistic (\texttt{sBoN}/\texttt{BoN}), pessimistic (\texttt{ITP}), and adaptive (\texttt{BoT}) strategies on MATH across different reference models and proxy rewards. 
    Each curve traces the trajectory of True Reward ($r^*$) vs. Proxy Reward ($\hat{r}$) as the sample size $N$ increases from $2^0$ to $2^{10}$, with marker size proportional to $N$. The smallest marker represents $N=1$ (standard sampling), while the largest represents $N=1024$.
    }
    \label{fig:app:math-all-acc-vs-reward}
  \end{center}
\end{figure*}

 \clearpage
\begin{figure*}[h!]
  \begin{center}
    \centerline{\includegraphics[width=.9\textwidth]{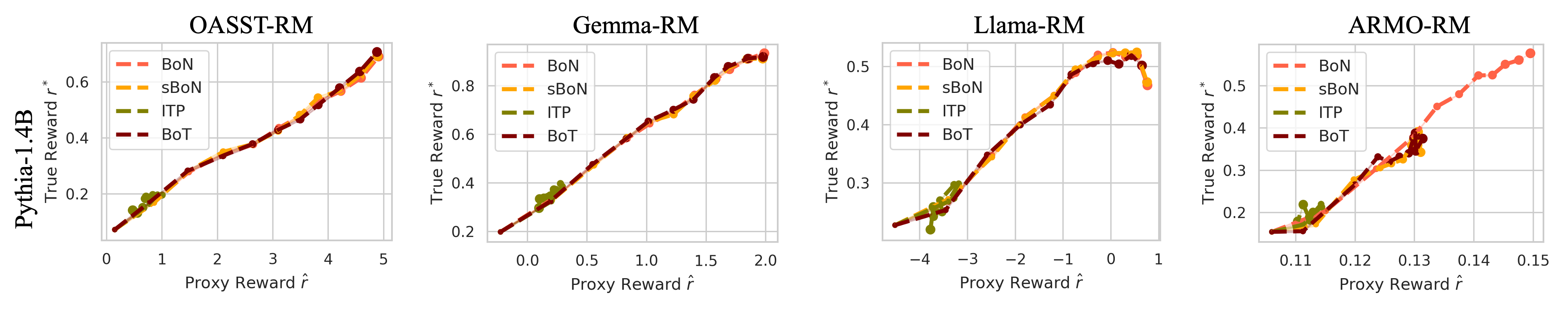}}
    \caption{\small 
    Comparison of optimistic (\texttt{sBoN}/\texttt{BoN}), pessimistic (\texttt{ITP}), and adaptive (\texttt{BoT}) strategies on AlpacaFarm across different reference models and proxy rewards. 
    Each curve traces the trajectory of True Reward ($r^*$) vs. Proxy Reward ($\hat{r}$) as the sample size $N$ increases from $2^0$ to $2^{10}$, with marker size proportional to $N$. The smallest marker represents $N=1$ (standard sampling), while the largest represents $N=1024$.
    }
    \label{fig:app:alpaca-all-acc-vs-reward}
  \end{center}
\end{figure*}

\clearpage
\subsection{Additional Results on Selected Configurations in Figures~\ref{fig:main-exp}}\label{app:sec:additional-exp}
We present the trajectory of true reward $r^*$ and proxy reward $\hat{r}$ versus the sample size $N$ increases (top row), alongside the trade-off (bottom row) between proxy reward and policy distortion (KL, $\chi^2$, and Tsallis divergences) on GSM8K (Figure~\ref{fig:app:gsm8k-addition}), MMLU (Figure~\ref{fig:app:mmlu-addition}), and MATH (Figure~\ref{fig:app:math-addition}).

While \texttt{BoN} maximizes the raw proxy reward, it suffers from reward hacking at large $N$, leading to lower true utility (accuracy). 
In contrast, \texttt{BoT} achieves the highest true reward among all strategies, surpassing both the conservative \texttt{ITP} and the optimistic baselines. 
The distortion analysis further validates the efficiency of \texttt{BoT}: it attains a higher proxy reward than \texttt{ITP} for the same level of distortion, yet maintains significantly tighter control over divergence than \texttt{sBoN} or \texttt{BoN} as $N$ scales, effectively balancing exploration with robustness.

\begin{figure*}[b!]
  \begin{center}
    \centerline{\includegraphics[width=.8\textwidth]{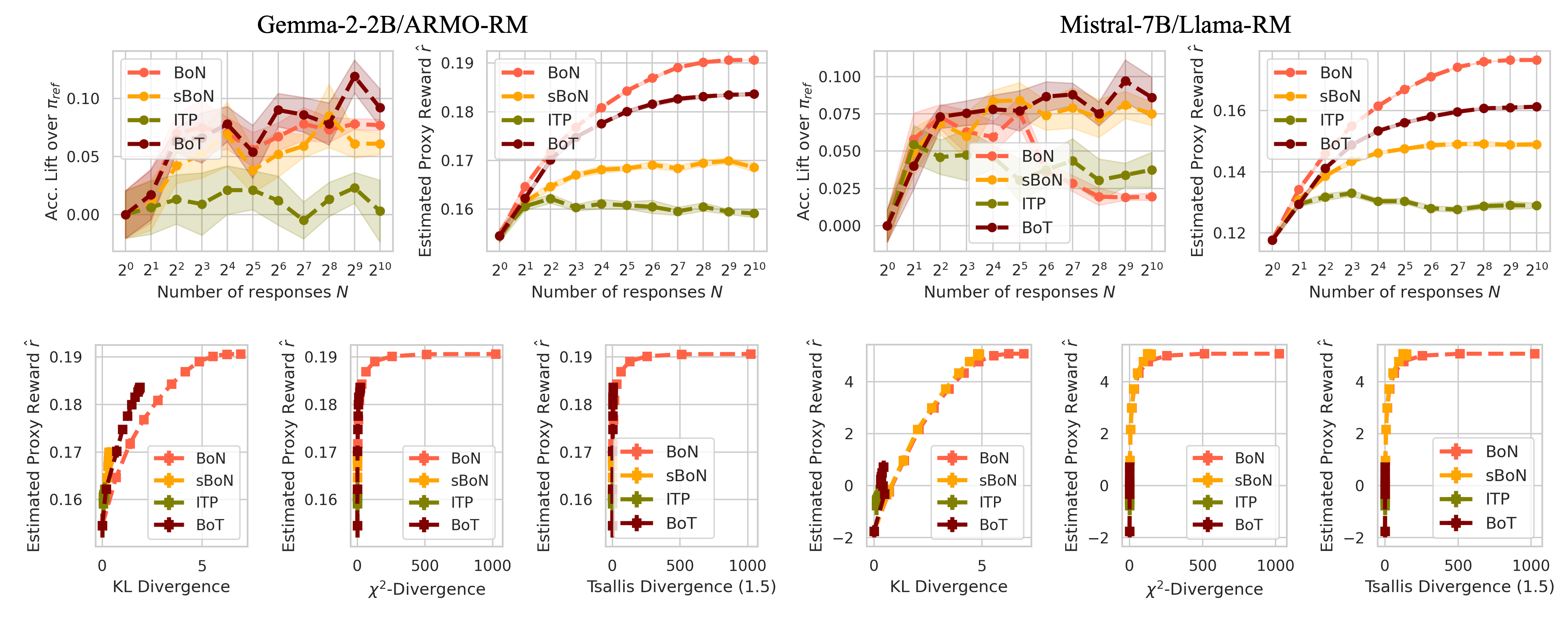}}
    \caption{\small 
    The trajectory of true reward $r^*$ and proxy reward $\hat{r}$ versus the sample size $N$ increases (top row), alongside the trade-off (bottom row) between proxy reward and policy distortion (KL, $\chi^2$, and Tsallis divergences) on GSM8K. 
    We compare optimistic (\texttt{sBoN}, \texttt{BoN}), pessimistic (\texttt{ITP}), and adaptive (\texttt{BoT}) strategies using select configurations from Figure~\ref{fig:main-exp}.
    }
    \label{fig:app:gsm8k-addition}
  \end{center}
\end{figure*}

\begin{figure*}[b!]
  \begin{center}
    \centerline{\includegraphics[width=.8\textwidth]{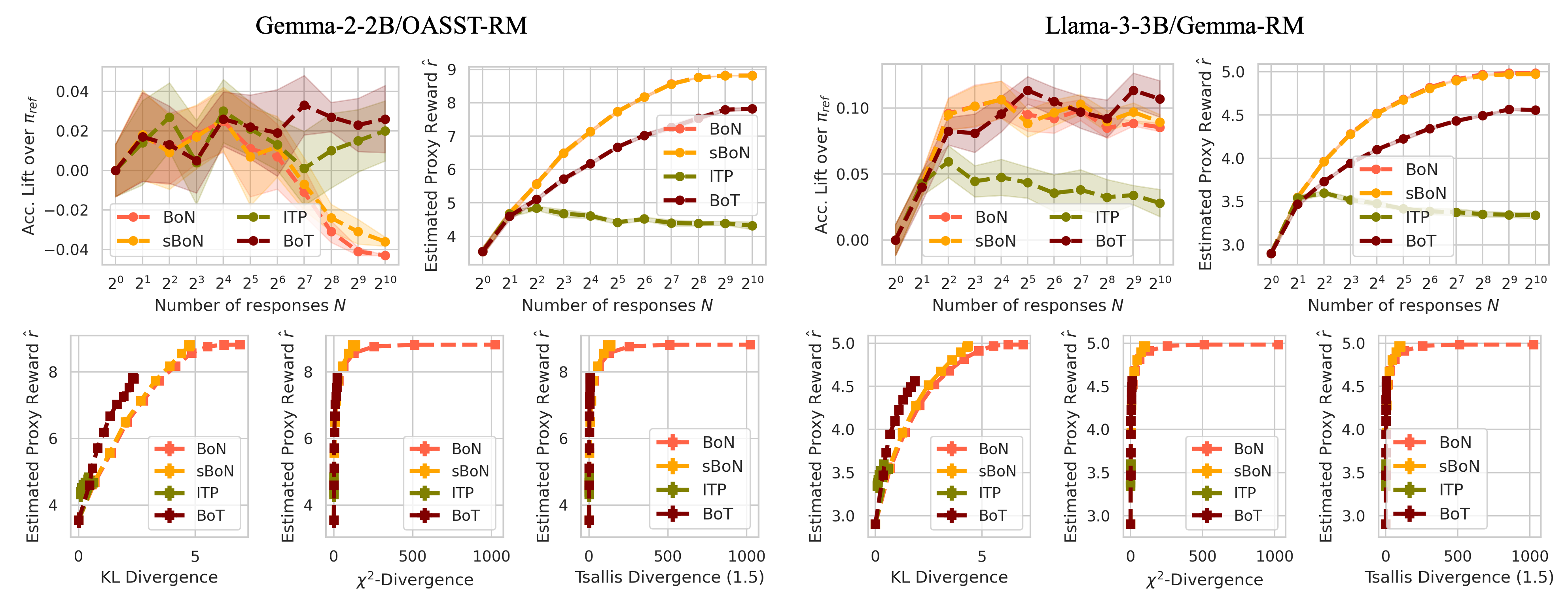}}
    \caption{\small 
    The trajectory of true reward $r^*$ and proxy reward $\hat{r}$ versus the sample size $N$ increases (top row), alongside the trade-off (bottom row) between proxy reward and policy distortion (KL, $\chi^2$, and Tsallis divergences) on MMLU. 
    We compare optimistic (\texttt{sBoN}, \texttt{BoN}), pessimistic (\texttt{ITP}), and adaptive (\texttt{BoT}) strategies using select configurations from Figure~\ref{fig:main-exp}.
    }
    \label{fig:app:mmlu-addition}
  \end{center}
\end{figure*}

\clearpage
\begin{figure*}[h!]
  \begin{center}
    \centerline{\includegraphics[width=.8\textwidth]{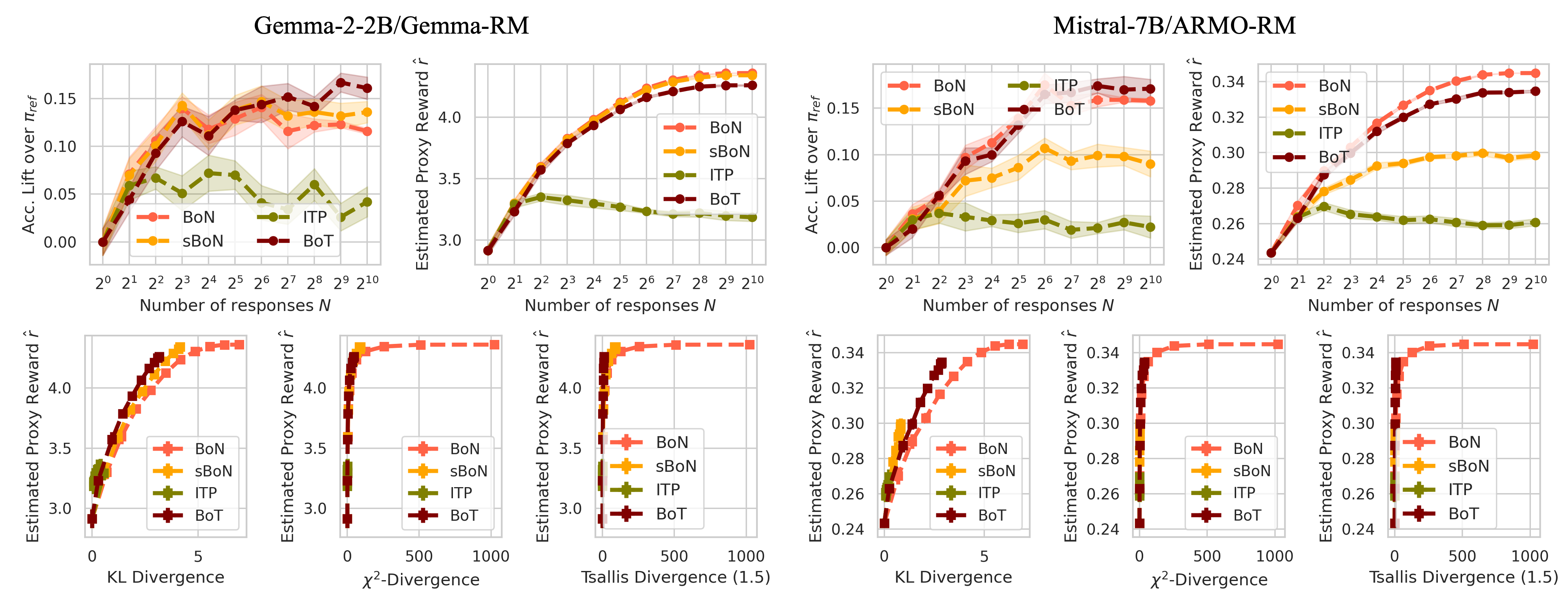}}
    \caption{\small 
    The trajectory of true reward $r^*$ and proxy reward $\hat{r}$ versus the sample size $N$ increases (top row), alongside the trade-off (bottom row) between proxy reward and policy distortion (KL, $\chi^2$, and Tsallis divergences) on MATH. 
    We compare optimistic (\texttt{sBoN}, \texttt{BoN}), pessimistic (\texttt{ITP}), and adaptive (\texttt{BoT}) strategies using select configurations from Figure~\ref{fig:main-exp}.
    }
    \label{fig:app:math-addition}
  \end{center}
\end{figure*}

\clearpage
\subsection{Ablation Studies}\label{app:sec:ablation}
We conduct ablation studies to analyze the sensitivity of the steering temperature $\lambda$ and the tail pivot $\kappa_0$ on the GSM8K dataset, utilizing \texttt{Gemma-2-2B} as the reference policy and \texttt{ARMO-RM} as the reward model. 
For all experiments in this section, we fix the sample size at $N=1024$.

\paragraph{Steering Temperature $\lambda$ (Figure~\ref{fig:app:gsm8k-gemma-armo-ablation}, Left).}
We first fix $\kappa_0$ to the median of $\hat{\kappa}(x)$ across prompts and sweep $\lambda \in \{0.001, \dots, 1.0\}$.
In the high-temperature regime (large $\lambda$), the re-weighting term $\exp(r/\lambda)$ flattens, causing the candidate distribution to approach uniformity; consequently, the behaviors of \texttt{sBoN}, \texttt{ITP}, and \texttt{BoT} converge.
As $\lambda$ decreases, the strategies diverge in how they prioritize high-reward responses. 
\texttt{ITP} exhibits stable performance due to its linear constraint, whereas \texttt{sBoN} becomes increasingly aggressive, converging toward the hard-maximization behavior of standard \texttt{BoN} (i.e., placing all weight on the single highest-reward candidate).
Crucially, while \texttt{BoT} does not always maximize the raw proxy reward—avoiding the over-optimization trap—its adaptive $\alpha(x)$ allows it to achieve higher \emph{true} reward (accuracy) by mitigating tail risks that optimistic baselines ignore.

\paragraph{Tail Pivot $\kappa_0$ (Figure~\ref{fig:app:gsm8k-gemma-armo-ablation}, Right).}
Next, we fix $\lambda=0.01$ and sweep $\kappa_0 \in \{0.001, \dots, 5.0\}$.
Since \texttt{sBoN} and \texttt{ITP} are independent of $\kappa_0$, they appear as constant baselines. 
The parameter $\kappa_0$ governs the transition point of our adaptive strategy:
\begin{itemize}
\item \textbf{Optimistic Limit ($\kappa_0 \to \infty$):} 
As $\kappa_0$ increases, the term $\frac{\hat{\kappa}(x)}{\hat{\kappa}(x)+\kappa_0}$ vanishes, driving $\alpha(x) = 1+\frac{\hat{\kappa}(x)}{\hat{\kappa}(x)+\kappa_0} \to 1$. In this regime, \texttt{BoT} mimics the optimistic behavior of \texttt{sBoN}, as reflected in the proxy reward curves.
\item \textbf{Pessimistic Limit ($\kappa_0 \to 0$):} 
Conversely, as $\kappa_0$ decreases below the median of $\hat{\kappa}(x)$, the ratio approaches 1, driving $\alpha(x) = 1+\frac{\hat{\kappa}(x)}{\hat{\kappa}(x)+\kappa_0} \to 2$. Here, \texttt{BoT} shifts toward the conservative, pessimistic profile of \texttt{ITP}.
\end{itemize}
The results demonstrate that by selecting an intermediate $\kappa_0$ (calibrated to the median tail index), \texttt{BoT} successfully interpolates between these regimes, outperforming static baselines by adapting to the specific risk profile of each prompt.

\begin{figure*}[h!]
  \begin{center}
    \centerline{\includegraphics[width=\textwidth]{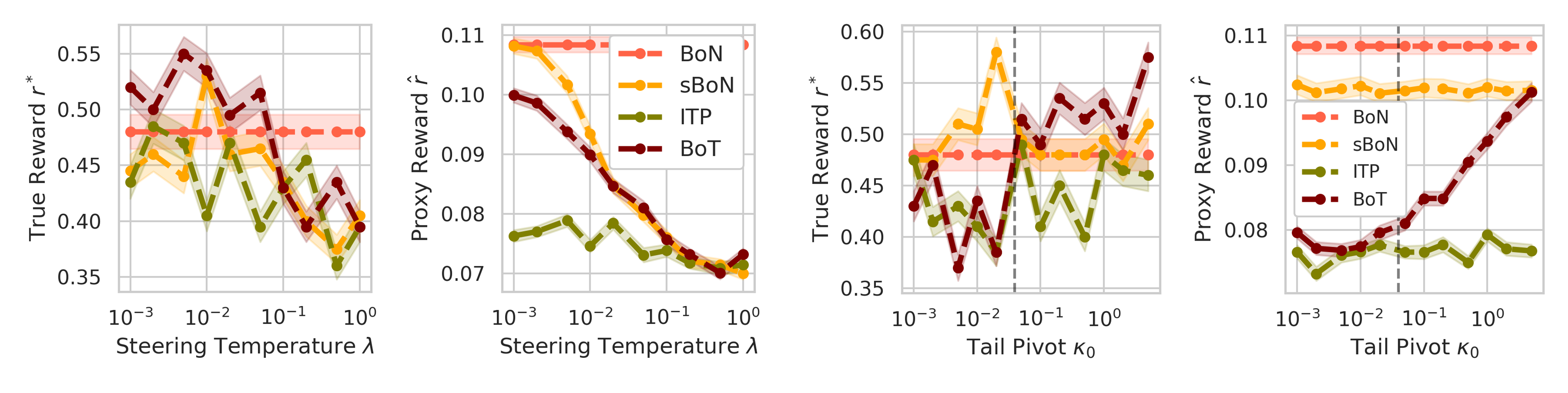}}
    \caption{\small 
    Ablation analysis of steering temperature $\lambda$ (\textbf{left two figures}) and tail pivot $\kappa_0$ (\textbf{right  two figures}). We compare optimistic (\texttt{sBoN}, \texttt{BoN}), pessimistic (\texttt{ITP}), and adaptive (\texttt{BoT}) strategies. The black dashed line marks the median of the estimated tail indices $\hat{\kappa}(x)$ across all prompts.
    }
    \label{fig:app:gsm8k-gemma-armo-ablation}
  \end{center}
\end{figure*}

\end{document}